\documentclass{article}
\usepackage{microtype}
\usepackage{graphicx}
\usepackage{subfigure}
\usepackage{booktabs} 
\usepackage{amsmath}
\usepackage{mathtools}
\usepackage{amssymb,amsfonts}
\usepackage{hyperref}
\usepackage{authblk}
\usepackage{algorithm}
\usepackage{algorithmic}
\makeatletter
\newcommand{\rmnum}[1]{\romannumeral #1}
\newcommand{\Rmnum}[1]{\expandafter\@slowromancap\romannumeral #1@}
\makeatother
\newtheorem{theorem}{Theorem}
\newtheorem{lemma}{Lemma}
\newtheorem{claim}{Claim}
\newtheorem{definition}{Definition}

\newtheorem{assumption}{Assumption}
\newtheorem{remark}{Remark}
\newenvironment{proof}{{\noindent\it Proof}.}{\hfill $\square$\par} 

\def\nf#1{{\color{green} #1}}
\newcommand{\hide}[1]{}
\title{
    A Novel Sequential Coreset Method for  Gradient Descent Algorithms}
\author[1]{\textbf{Jiawei Huang}}

\author[2]{\textbf{Ruomin Huang}}

\author[1]{\textbf{Wenjie Liu}}

\author[1]{\textbf{Nikolaos M. Freris}}

\author[1]{\textbf{Hu Ding\thanks{Corresponding author.}}}

\affil[1]{School of Computer Science and Technology} \affil[2]{School of Data Science}
\affil[ ]{University of Science and Technology of China}
\affil[ ]{\texttt{\{hjw0330, hrm, lwj1217\}}\texttt{@mail.ustc.edu.cn}, \texttt{\{nfr,\href{mailto:huding@ustc.edu.cn}{huding\}@ustc.edu.cn}}}

\begin{document}
\date{}
\maketitle

\begin{abstract}
A wide range of optimization problems arising in machine learning can be solved by gradient descent algorithms, and a central question in this area is how to efficiently compress a large-scale dataset so as to reduce the computational complexity. {\em Coreset} is a popular data compression technique that  has been extensively studied before. However, most of existing coreset methods are problem-dependent and cannot be used as a general tool for a broader range of applications. A key obstacle is that they often rely on the pseudo-dimension and total sensitivity bound that can be very high or hard to obtain. In this paper, based on the ``locality'' property of gradient descent algorithms, we propose a new framework, termed ``sequential coreset'', which effectively avoids these obstacles. Moreover, our method is particularly suitable for sparse optimization whence the coreset size can be further reduced to be only poly-logarithmically dependent on the dimension. In practice, the experimental results suggest that our method can save a large amount of running time compared with the baseline algorithms.
\end{abstract}

\section{Introduction}
\label{sec-intro}

{\em Coreset}~\cite{DBLP:journals/widm/Feldman20} is a popular technique for compressing large-scale datasets so as to speed up existing algorithms. Especially for the optimization problems arising in machine learning, coresets have been extensively studied in recent years. Roughly speaking, given a large dataset $P$ and a specified optimization objective ({\em e.g.,} $k$-means clustering), the coreset approach is to construct a new dataset $\tilde{P}$ with the size $|\tilde{P}|\ll |P|$, such that any solution obtained over $\tilde{P}$ will approximately preserve the same quality over the original set $P$; that is, we can replace $P$ by $\tilde{P}$ when running an available algorithm for solving this optimization problem. Because $|\tilde{P}|\ll |P|$, the runtime can be significantly reduced.

In this paper, we consider {\em Empirical Risk Minimization (ERM)} problems which capture a broad range of applications in machine learning~\cite{DBLP:conf/nips/Vapnik91}. Let $\mathbb{X}$ and $\mathbb{Y}$ be the data space and response space, respectively. 
Given an input training set $P=\{(x_1, y_1), (x_2, y_2), \cdots, (x_n, y_n)\}$, where each $x_i\in \mathbb{X}$ and each $y_i\in \mathbb{Y}$, 
the objective is to learn the hypothesis $\beta$ (from the hypothesis space $\mathbb{R}^d$) so as to minimize the {\em empirical risk} 
\begin{eqnarray}
F(\beta)=\frac{1}{n}\sum^n_{i=1}f(\beta, x_i, y_i), \label{for-er}
\end{eqnarray}
where $f(\cdot, \cdot, \cdot)$ is the non-negative real-valued \emph{loss function}. In practice, the data size $n$ can be very large, thus it is instrumental to consider data compression methods (like coresets) to reduce the computational complexity.

Let $\epsilon\in [0,1]$. A standard \textbf{$\epsilon$-coreset} is represented as a vector $W=[w_1, w_2, \cdots, w_n]\in\mathbb{R}^n$ 
with the property that the function $\tilde{F}(\beta)=\frac{1}{\sum^n_{i=1}w_i}\sum^n_{i=1}w_i f(\beta, x_i, y_i)$ must satisfy 
\begin{eqnarray}
\tilde{F}(\beta)\in (1\pm\epsilon)F(\beta), \ \forall \beta\in\mathbb{R}^d. \label{for-orcoreset}
\end{eqnarray}
The number of non-zero entries $\|W\|_0$ is the coreset size, and thus 
the goal of compression is to have $W$ to be as sparse as possible. 
Suppose we have an algorithm that can achieve $c$-approximation for the ERM problem ($c\geq 1$). Then, we can run the same algorithm on the $\epsilon$-coreset, and let $\hat{\beta}$ be the returned $c$-approximation, {\em i.e.,} $\tilde{F}(\hat{\beta})\leq c\cdot \min_{\beta\in\mathbb{R}^d}\tilde{F}(\beta)$. It holds that $F(\hat{\beta})\leq c\cdot\frac{1+\epsilon}{1-\epsilon}\min_{\beta\in\mathbb{R}^d} F(\beta)$, 

thus the approximation ratio of $\hat{\beta}$ for the original objective function $F(\cdot)$ is only slightly worse than $c$ when $\epsilon$ is sufficiently small.

A large part of coreset methods are based on the ``sensitivity'' idea~\cite{langberg2010universal}.  First, it computes a constant factor approximation with respect to the objective function (\ref{for-er}); 
then it estimates the sensitivity $\sigma_i$ for each data item $(x_i, y_i)$ based on the obtained constant factor approximation; finally, it takes a random sample (as the coreset) over the input set $P$, where each data item $(x_i, y_i)$ is selected with probability proportional to its sensitivity $\sigma_i$, and the total sample size depends on the total sensitivity bound $\sum^n_{i=1}\sigma_i$ along with  the ``pseudo-dimension'' of the objective function~\cite{DBLP:conf/stoc/FeldmanL11,li2001improved}. This sensitivity-based coreset framework has been  successfully applied to solve  problems such as $k$-means clustering and projective clustering~\cite{DBLP:conf/stoc/FeldmanL11}. 
However, there are several obstacles when trying to apply this approach to general ERM problems. For instance, it is not easy to obtain a constant factor approximation; moreover, different from clustering problems, it is usually challenging to achieve a reasonably low total sensitivity bound and compute the pseudo-dimension for many practical ERM problems. For example, the coreset size can be as large as $\tilde{\Omega}(d^2\sqrt{n}/\epsilon^2)$ for logistic regression~\cite{DBLP:conf/nips/TukanMF20} with $O(nd^2)$ construction time.

Another common class of coreset construction methods is based on ``greedy selection''~\cite{DBLP:conf/iclr/ColemanYMMBLLZ20,DBLP:conf/icml/MirzasoleimanBL20}. The greedy selection procedure is quite similar to the {\em $k$-center clustering} algorithm~\cite{G85} and the greedy {\em submodular set cover} algorithm~\cite{DBLP:journals/combinatorica/Wolsey82}. Intuitively, the method greedily selects a subset of the input training set, {\em i.e.,} the coreset, which are expected to be as diverse as possible; consequently, the whole training set can be covered by small balls centered at the selected subset. Nonetheless, this approach also suffers from several drawbacks. First, it is difficult to bound the size of the obtained coreset, when specifying the error bound induced by the coreset ({\em e.g.,} one may need too many balls to cover the training set if their radii are required to be no larger than an upper bound). Second, the time complexity can be too high, {\em e.g.,} the greedy $k$-center clustering procedure usually needs to read the input training set for a large number of passes, and the greedy submodular set cover algorithm usually needs a large number of function evaluations. 
\subsection{Our Contributions}
\label{sec-our}

The aforementioned issues seriously limit the applications of coresets in practice. 
In this paper, we propose a novel and easy-to-implement coreset framework, termed {\em sequential coreset},  for the general ERM problem (\ref{for-er}). 
Our idea comes from a simple observation. For many ERM problems, either convex or non-convex, gradient descent algorithms are commonly invoked. In particular, these gradient descent algorithms usually share the following \textbf{locality property}:

{\em \hspace{2em}Since the learning rate of a gradient descent algorithm is usually restricted by an upper bound, the trajectory of the hypothesis $\beta$ in (\ref{for-er}) is likely to be ``smooth'' (except for the first few  rounds). That is, the change of $\beta$ should be ``small'' between successive rounds. }

This allows to focus, in each round, on a local region rather than the whole hypothesis space $\mathbb{R}^d$. We can thus visualize the trajectory to be decomposed into a sequence of ``segments'', where each segment is bounded by an individual ball. See Figure~\ref{fig-local} for an illustration. When the trajectory enters a new ball ({\em i.e.,} a new local region), we construct a coreset $W=[w_1, w_2, \cdots, w_n]$ with
\begin{eqnarray}
\tilde{F}(\beta)\in (1\pm\epsilon)F(\beta), \forall \beta\in\text{the current local region}. \label{for-loccoreset}
\end{eqnarray}
The formal definition of such a ``local'' coreset is shown in Section~\ref{sec-pre}. When the trajectory approaches the boundary of the ball, we update the coreset for the next ball. Therefore, we call the method ``sequential coreset''.

\begin{figure}[]
    \centering
  \includegraphics[height=1.7in]{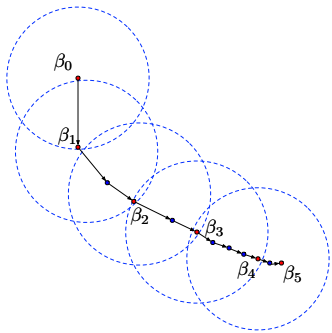}
      \caption{Suppose the trajectory starts from $\beta_0$, and we construct a ``local'' coreset within the ball centered at $\beta_0$; when the trajectory approaches the ball's boundary, {\em i.e.,} $\beta_1$, we update the coreset. Similarly, we update the coreset at $\beta_2$, $\beta_3$, and $\beta_4$,  until sufficiently approximating the stationary point $\beta_5$. We can view $\beta_1$, $\beta_2$, $\beta_3$, and $\beta_4$ as a sequence of ``anchors''. }
  \label{fig-local}
\end{figure}

Although building the coreset for a local region is easier than that for the global hypothesis space, there remain several technical challenges to resolve. Partly inspired by the layered sampling idea of~\cite{chen2009coresets,DBLP:conf/icml/DingW20}, we can achieve a coreset of (\ref{for-loccoreset})  where the coreset size depends on the range of the local region. In particular, our method enjoys several significant advantages compared with  previous coreset methods: 
\begin{itemize}
\item Our method is not problem-dependent and can be applied to any (convex or non-convex) ERM problem that uses gradient descent, under some mild assumptions. In fact, our method can be extended to apply to other iterative algorithms beyond gradient descent, such as subgradient descent and expectation maximization, as long as they satisfy the locality property.

\item Our method can avoid  to compute the total sensitivity bound and pseudo-dimension, thus it does not incur any complicated computations ({\em e.g.,} SVD) and has only linear construction time. 

\item For special cases of practical interest such as sparse optimization, the coreset size can be further reduced to be only poly-logarithmically dependent on the dimension. 

\end{itemize}
\subsection{Related Works}
\label{sec-related}

\textbf{Gradient descent.} Given a differentiable objective function, gradient descent is arguably the most common first-order iterative optimization algorithm for finding the optimal solution~\cite{gdcurry}. A number of ERM models can be solved via gradient descent methods, such as Ridge regression~\cite{ridge} and Logistic regression~\cite{CRAMER2004613}. Note that though the objective function of the Lasso regression~\cite{10.2307/2346178} is not differentiable, several natural generalizations of the traditional gradient descent method, such as subgradient methods~\cite{subgradient} and proximal gradient methods~\cite{DBLP:conf/pkdd/MosciRSVV10,FISTA}, have been developed and shown to perform well in practice. 

Several extensions of gradient descent have been also widely studied in recent years. For example,  Nesterov introduced the acceleration technique for achieving faster gradient method~\cite{nesterov1983method}.  
  
In view of the rapid development of deep learning and many other machine learning applications, stochastic gradient descent method  and variants have played a central role in the field of large-scale machine learning, due to their scalability to very large, possibly distributed datasets~\cite{DBLP:journals/siamrev/BottouCN18,DBLP:journals/corr/KingmaB14,DBLP:journals/jmlr/DuchiHS11}.

\textbf{Coresets.} Compared with other data compression approaches, an obvious advantage of coreset is that it is to be selected from the original input; that is, the obtained coreset can well preserve some favorable properties, such as  sparsity and interpretability, in the input domain. In the past years, coreset techniques have been widely applied to many optimization problems, such as:  clustering~\cite{chen2009coresets,DBLP:conf/stoc/FeldmanL11,huang2018epsilon}, logistic regression~\cite{huggins2016coresets,munteanu2018coresets,DBLP:conf/aistats/SamadianPMIC20,DBLP:conf/nips/TukanMF20,DBLP:conf/aistats/SamadianPMIC20}, Bayesian methods~\cite{DBLP:conf/icml/CampbellB18,DBLP:conf/nips/CampbellB19}, linear regression~\cite{dasgupta2009sampling,drineas2006sampling,DBLP:conf/icml/Chhaya0S20,DBLP:conf/aistats/KachamW20,DBLP:conf/nips/TukanMF20}, robust optimization~\cite{DBLP:conf/icml/DingW20}, Gaussian mixture model~\cite{lucic2017training}, and active learning~\cite{DBLP:conf/iclr/ColemanYMMBLLZ20,DBLP:conf/iclr/SenerS18}. Recently, \cite{DBLP:conf/nips/MaaloufJF19} also proposed the notion of ``accurate'' coresets, which do not introduce any approximation error when compressing the input dataset. 
Coresets are also applied to speed up large-scale or distributed machine learning algorithms~\cite{DBLP:conf/uai/ReddiPS15,DBLP:conf/icml/MirzasoleimanBL20,DBLP:conf/nips/MirzasoleimanCL20,DBLP:conf/nips/BorsosM020}. 

Very recently, \cite{DBLP:conf/aistats/RajMM20} also considered the ``local'' heuristic for coresets. However, their results are quite different from ours. Their method still relies on the problem-dependent pseudo-dimension and the sensitivities; moreover, their method requires the objective function to be strongly convex.

\textbf{Sketch.} Another widely used data summarization method is ``sketch''~\cite{DBLP:journals/corr/Phillips16}. 
 Different from coresets, sketch does not require to generate the summary from the original input. The sketch technique is particularly popular for solving linear regression problems (with and without regularization)~\cite{DBLP:conf/approx/AvronCW17,DBLP:conf/icml/ChowdhuryYD18}.

\section{Preliminaries}
\label{sec-pre}
Given an instance of the ERM problem (\ref{for-er}), we assume that the loss function is Lipschitz smooth. This is a quite common assumption for analyzing many gradient descent methods~\cite{acsent}.

\begin{assumption}[Lipschitz Smoothness]
\label{assump-lip}
There exists  a real constant $L >0$, such that for any $1\leq i\leq n$ and any $\beta_1, \beta_2 $ in the hypothesis space, we have
	\begin{equation}\label{eq:Lips}
	\|\nabla f(\beta_1,x_i, y_i) - \nabla f(\beta_2,x_i, y_i)\| \le L\|\beta_1 - \beta_2\|, 
	\end{equation}
	where  $\|\cdot\|$ is the Euclidean norm in the space. 

\end{assumption}

For simplicity, we just use ball to define the local region for constructing our coreset as (\ref{for-loccoreset}). Suppose we have the ``anchor'' $\beta_{\mathtt{anc}}\in \mathbb{R}^d$ and region range $R\geq 0$. Let $\mathbb{B}(\beta_{\mathtt{anc}}, R)$ denote the ball centered at $\beta_{\mathtt{anc}}$ with radius $R$. Below, we provide the formal definition for the ``local'' coreset  in (\ref{for-loccoreset}). 
 
\begin{definition}[Local $\epsilon$-Coreset]
\label{def-local}
Let $P=\{(x_1, y_1),$ $ (x_2, y_2),$$ \cdots, (x_n, y_n)\}$ be an input dataset of the ERM problem (\ref{for-er}). Suppose $\epsilon\in (0,1)$. Given $\beta_{\mathtt{anc}}\in \mathbb{R}^d$ and $R\geq 0$, the local $\epsilon$-coreset, denoted $\mathcal{CS}_\epsilon(\beta_{\mathtt{anc}}, R)$, is a vector $W=[w_1, w_2, \cdots, w_n]$ satisfying that
\begin{eqnarray}
\tilde{F}(\beta)\in (1\pm\epsilon)F(\beta), \hspace{0.2in}\forall \beta\in\mathbb{B}(\beta_{\mathtt{anc}}, R), \label{for-loccoreset2}
\end{eqnarray}
where $\tilde{F}(\beta)=\frac{1}{\sum^n_{i=1}w_i}\sum^n_{i=1}w_i f(\beta, x_i, y_i)$. The number of non-zero entries of $W$ is the size of $\mathcal{CS}_\epsilon(\beta_{\mathtt{anc}}, R)$.
\end{definition}

\section{Local $\epsilon$-Coreset Construction}
\label{sec-construction}

We first present the construction algorithm for local $\epsilon$-coreset, and expose the detailed analysis on its quality in Section~\ref{sec-ana}. Besides the quality guarantee of (\ref{for-loccoreset2}), in Section~\ref{sec-grad} we show that our coreset can approximately preserve the gradient $ \nabla F (\beta)$, which is an important property for  gradient descent algorithms. In Section~\ref{sec-beyond}, we discuss some extensions beyond gradient descent. Relying on the local $\epsilon$-coreset, we propose the sequential coreset framework and consider several important applications in Section~\ref{sec-app}. 

\begin{algorithm}[tb]
	\caption{Local $\epsilon$-Coreset Construction}
	\label{alg: local coreset}
	\begin{algorithmic}
		\STATE {\bfseries Input:} A training dataset $P=\{(x_1, y_1),$ $ (x_2, y_2),  \cdots, (x_n, y_n)\}$, the Lipschitz constant $L$ as described in Assumption~\ref{assump-lip}, $\beta_{\mathtt{anc}}\in \mathbb{R}^d$, and the  parameters $R\geq 0$ and $\epsilon\in (0,1)$.
		\begin{enumerate}
		\item Let $N=\lceil\log n\rceil$ and $H=F(\beta_{\mathtt{anc}})$; initialize $W=[0, 0, \cdots, 0]\in \mathbb{R}^n$.
		\item The set $P$ is partitioned into $N+1$ layers $\{P_{0}, \hdots, P_N\}$ as in (\ref{for-layer1}) and (\ref{for-layer2}). 
		\item For each $P_j\ne\emptyset$, $0\leq j\leq N$:
		\begin{enumerate}
		\item  take a random sample $Q_j$ from $ P_j$ uniformly at random, where the size $|Q_j|$ depends on the parameters $\epsilon$, $R$, and $L$ (the exact value will be discussed in our following analysis in Section~\ref{sec-ana});
		\item  for each sampled data item $(x_i, y_i)\in Q_j$,  assign the weight  $w_i=\frac{|P_j|}{|Q_j|}$;
		\end{enumerate}
		\end{enumerate}

		\STATE {\bfseries Output:} the weight vector $W=[w_1, w_2, \cdots, w_n]$ as the coreset.
	\end{algorithmic}
\end{algorithm}

\textbf{Coreset construction.} Let $N=\lceil \log n \rceil$ (the basis of the logarithm is 2 in this paper). Given the central point $\beta_{\mathtt{anc}}\in\mathbb{R}^d$ and the local region range ({\em i.e.,} the radius) $R\geq0$, we set $H=F(\beta_{\mathtt{anc}})$ and then partition the input dataset  $P=\{(x_1, y_1),$ $ (x_2, y_2),  \cdots, (x_n, y_n)\}$ into $N+1$ layers: 
\begin{eqnarray}
 P_0&=&\big\{(x_i, y_i)\in P\mid f(\beta_{\mathtt{anc}}, x_i, y_i) \le H\big\}, \label{for-layer1}\\
P_j&=&\big\{(x_i, y_i)\in P\mid 2^{j-1}H < f(\beta_{\mathtt{anc}}, x_i, y_i) \le\nonumber\\
&& \hspace{0.4in} 2^{j}H\big\}, 1 \le j \le N.\label{for-layer2}
 \end{eqnarray}	
 It is easy to see that $P=\cup^N_{j=0}P_j$, since $f(\beta_{\mathtt{anc}}, x_i, y_i)$ is always no larger than $2^{N}H$ for any $1\leq i\leq n$. 
For each $0\leq j\leq N$, if $P_j\ne\emptyset$, we take a random sample $Q_j$ from $ P_j$ uniformly at random, where the size $|Q_j|$ will be determined in our following analysis (in Section~\ref{sec-ana});  
for each sampled data item $(x_i, y_i)\in Q_j$, we assign the weight  $w_i$ to be   $\frac{|P_j|}{|Q_j|}$; for all the data items of $P_j\setminus Q_j$, we let their weights to be $0$. 
At the end, we obtain the weight vector $W=[w_1, w_2, \cdots, w_n]$ as our coreset, and consequently $\tilde{F}(\beta)=\frac{1}{n}\sum^n_{i=1}w_i f(\beta, x_i, y_i)$ (it is easy to verify $\sum^n_{i=1}w_i=n$ from our construction). The construction procedure is shown in Algorithm~\ref{alg: local coreset}. 
\begin{remark}
\label{rem-construction}
Our layered sampling procedure in Algorithm~\ref{alg: local coreset} is similar to the coreset construction idea of \cite{chen2009coresets,DBLP:conf/icml/DingW20}, which was originally designed for the $k$-median/means clustering problems. Compared with the sensitivity based coreset construction idea~\cite{langberg2010universal}, a significant advantage of our method is that there is no need to compute the total sensitivity bound and  pseudo-dimension. These values are problem-dependent and, for some objectives, they can be very high or hard to obtain~\cite{munteanu2018coresets,DBLP:conf/nips/TukanMF20}.  
 
\end{remark}

\subsection{Theoretical Analysis}
\label{sec-ana}
In this section, we prove the quality guarantee and complexity of the coreset returned from Algorithm~\ref{alg: local coreset}. We define two values before presenting our theorem,  
$M\coloneqq \max\limits_{1\leq i\leq n}\|\nabla f(\beta_{\mathtt{anc}},x_i,y_i)\|$ and $m\coloneqq \!\!\! \min\limits_{\beta \in \mathbb{B}(\beta_{\mathtt{anc}},R)}F(\beta)$.

	\begin{theorem}
	\label{local_coreset}
	With problability $1-\frac{1}{n}$, Algorithm \ref{alg: local coreset} returns a qualified coreset  $\mathcal{CS}_\epsilon(\beta_{\mathtt{anc}}, R)$ with size  $\tilde{O}\left(\left(\frac{H+MR+LR^2}{m}\right)^2\cdot\frac{d}{\epsilon^2}\right)$\footnote{
		$\tilde{O}(g)\coloneqq O\left(g\cdot \mathtt{polylog}\left(\frac{nHMR}{\epsilon m}\right)\right)$.}. Furthermore, when the vector $\beta$ is restricted to have at most $k\in \mathbb{Z}^+$ non-zero entries in the hypothesis space $\mathbb{R}^d$, the coreset size can be reduced to be $\tilde{O}\left(\left(\frac{H+MR+LR^2}{m}\right)^2\cdot\frac{k\log d}{\epsilon^2}\right)$. The runtime of Algorithm \ref{alg: local coreset} is $O(n\cdot t_f)$, where $t_f$ is the time complexity for computing the loss $f(\beta, x, y)$. 	

\end{theorem}
\begin{remark}
\label{rem-the-quality}
From Theorem~\ref{local_coreset} we can see that the coreset size depends on the initial vector $\beta_{\mathtt{anc}}$ and the local region range $R$. Also note that  the value $m$ is non-increasing with $R$. 
\end{remark}
 First, the linear time complexity of Algorithm \ref{alg: local coreset} is easy to see: to obtain the partition and the samples, it just needs to compute $f(\beta_{\mathtt{anc}}, x_i, y_i)$ for $1\le i\le n$. 

Below, we  focus on proving the quality guarantee and coreset size.  For the sake of simplicity, we use $f_i(\beta)$ to denote $f(\beta,x_i,y_i)$ in our analysis. By using Taylor expansion and Assumption \ref{assump-lip}, we directly have 

\begin{eqnarray}\label{eq:Lip-bound}
f_i(\beta) 
&\in f_i(\beta_{\mathtt{anc}}) \pm \left(\frac{ \|\nabla f_i(\beta_{\mathtt{anc}})\|}{R}+\frac{L}{2}\right)R^2.\label{eqa: betabound}
\end{eqnarray}

for any $\beta \in \mathbb{B}(\beta_{\mathtt{anc}},R)$ and $1\leq i\leq n$. 
Then we have the following lemma. 
	\begin{lemma} 
	\label{lemma_hoeffding}
	We fix a vector  $\beta \in \mathbb{B}(\beta_{\mathtt{anc}},R)$ and an index $j$ from $\{0, 1, \hdots, N\}$. Given any two numbers $\lambda \in (0,1)$ 
	 and $\delta>0$, if we set the sample size  in Step 3(a) of Algorithm \ref{alg: local coreset} to be 
	\begin{eqnarray}\label{eq:coreset-size}
	|Q_j| = \ O\left((2^{j-1}H+MR+LR^2)^2\delta^{-2}\log\frac{1}{\lambda}\right), \label{for-hoeffding-qsize}
	\end{eqnarray}
we have

\begin{equation}
	\mathtt{Prob}\!\left[\left|\frac{1}{|{Q_j}|} \!\sum_{(x_i,y_i)\in {Q_j}} \!\! \!\! \! \! f_i(\beta)-\frac{1}{|P_j|} \!\sum_{(x_i,y_i)\in P_j}\!\!\!\! \!\! f_i(\beta)\right|\geq \delta\right] \!\!\leq\!\! \lambda.\nonumber
\end{equation}
	
\end{lemma}

\begin{proof}
For a fixed $1\le j \le N$, we view   $f_i(\beta)$ as an independent random variable for each  $(x_i, y_i)\in P_j$. Through the partition construction (\ref{for-layer1}) and (\ref{for-layer2}), and the bounds (\ref{eqa: betabound}),  we have 
\begin{equation}
 \left.\begin{aligned}
 f_i(\beta)&\geq&2^{j-1}H-MR- \frac{1}{2}LR^2;  \\
f_i(\beta)&\leq& 2^jH +MR+ \frac{1}{2}LR^2. 
               \end{aligned}
  \hspace{0.2in}\right  \} \label{for-uplow}
  \qquad  
\end{equation}

	Let the sample size $|Q_j|=\lceil \frac{1}{2}(2^{j-1}H+LR^2+2MR)^2\delta^{-2}$ln$\frac{2}{\lambda}\rceil$.
	Through the Hoeffding's inequality~\cite{hoeffding1994probability}, we know that 

	\begin{eqnarray}
	\mathtt{Prob}\left[\left|\frac{1}{|{Q_j}|} \!\sum_{(x_i,y_i)\in {Q_j}} \! \! \! \! f_i(\beta)-\frac{1}{|P_j|} \!\sum_{(x_i,y_i)\in P_j}\!\!\! \! f_i(\beta)\right|\geq \delta\right] \nonumber
	\end{eqnarray}
	 is no larger than $2e^{-\frac{2|Q_j|\delta^2}{(2^{j-1}H+LR^2+2MR)^2}}\leq \lambda$. 
	 
	  Now we consider the case $j=0$. For any data item  $(x_i,y_i)\in P_0$, we have $0\leq f_i(\beta)\leq H+MR+\frac{1}{2}LR^2$. 
	 If letting the sample size $|Q_0|=\lceil \frac{1}{2}(H+\frac{1}{2}LR^2+MR)^2\delta^{-2}$ln$\frac{2}{\lambda}\rceil$, it is easy to verify that the same probability bound also holds. 

\end{proof}

After proving Lemma~\ref{lemma_hoeffding}, we further show  $\tilde{F}(\beta)\approx F(\beta)$ for any fixed $\beta\in \mathbb{B}(\beta_{\mathtt{anc}},R)$. 

\begin{lemma}
	\label{lemma2}
	Suppose $\epsilon_1\geq 0$. In Lemma~\ref{lemma_hoeffding}, if we set $\delta=\epsilon_12^{j-1}H$ for $j=0, 1, \cdots, N$, then for any fixed $ \beta \in \mathbb{B}(\beta_{\mathtt{anc}},R)$, 
	\begin{eqnarray}
	\Big|\tilde{F}(\beta) - F(\beta)\Big| \leq \frac{3}{2}\epsilon_1 F(\beta_{\mathtt{anc}})
	\end{eqnarray}	
	holds with probability at least $1-(N+1)\lambda$. 

\end{lemma}
\begin{proof}

 From Lemma \ref{lemma_hoeffding}, it holds that the probability that 
      \begin{eqnarray}
      \left|\frac{|P_j|}{|{Q_j}|} \sum_{(x_i,y_i)\in {Q_j}}f_i(\beta)- \sum_{(x_i,y_i)\in P_j}f_i(\beta)\right|\nonumber\\
      \geq |P_j|\cdot\epsilon_1 2^{j-1}H   \label{for-lem2-1}   
      \end{eqnarray}
      is at most $\lambda$. Recall $\tilde{F}(\beta)=\frac{1}{n}\sum^n_{i=1}w_i f(\beta, x_i, y_i)$, where for each $(x_i,y_i)\in P_j$, $w_i=\frac{|P_j|}{|Q_j|}$ if $(x_i,y_i)\in Q_j$, and $w_i=0$ if $(x_i,y_i)\in P_j\setminus Q_j$.  Thus, by taking the union bound of (\ref{for-lem2-1}) over $0\leq j\leq N$, we have
      \begin{eqnarray}
     && n\left|\tilde F(\beta)-F(\beta)\right|\nonumber\\
      &=&\left|\sum_{j=0}^{N}\frac{|P_j|}{|Q_j|}\sum_{(x_i,y_i)\in Q_j}f_i(\beta)-\sum_{j=0}^{N}\sum_{(x_i,y_i)\in P_j}f_i(\beta)\right|\nonumber\\
      &\leq&\sum_{j=0}^{N}\left|\frac{|P_j|}{|Q_j|}\sum_{(x_i,y_i)\in Q_j}f_i(\beta)-\sum_{(x_i,y_i)\in P_j}f_i(\beta)\right|\nonumber\\
      &\leq& \sum_{j=0}^{N}|P_j|\epsilon_1 2^{j-1}H \label{for-lem2-2}
                    \end{eqnarray}
      with probability at least $(1-\lambda)^{N+1}>1-(N+1)\lambda$. To complete the proof, we also need the following claim.
	\begin{claim}
\label{claim_1}
	$\sum_{j=0}^{N}|P_j|2^j \leq 3n$.
\end{claim}
\begin{proof}
	By the  definition of $P_j$, we have
	\begin{equation}
 \begin{aligned}
       2^jH &=H\hide{\nf{\ge f_i(\beta_{\mathtt{anc}}), \forall (x_i,y_i)\in P_j}},  &\text{if $j=0$; }\\
        2^jH&\leq 2f_i(\beta_{\mathtt{anc}}), \forall (x_i,y_i)\in P_j,  &\text{if $ j\geq 1$.}
                \end{aligned}
  \qquad  
\end{equation}
 	Therefore, $2^jH $ is always no larger than $ 2f_i(\beta_{\mathtt{anc}})+H$ for any $0\leq j\leq N$ and any  $(x_i,y_i)\in P_j$.

	 Overall, 
	\begin{equation}
	\begin{split}
	\begin{aligned}
	\sum_{j=0}^{N}|P_j|2^jH&=\sum_{j=0}^{N}\sum_{(x_i,y_i)\in P_j}2^jH \\
	&\leq \sum_{j=0}^{N}\sum_{(x_i,y_i)\in P_j}(2f_i(\beta_{\mathtt{anc}})+H)\\
	&=2nF(\beta_{\mathtt{anc}}) + nH=3nH. 
	\end{aligned}
	\end{split}
	\end{equation}
	Thus the claim $\sum_{j=0}^{N}|P_j|2^j \leq 3n$  
	is true. 
\end{proof}

     By  using Claim~\ref{claim_1}, (\ref{for-lem2-2}) can be rewritten as 
      \begin{eqnarray}
      n\big|\tilde F(\beta)-F(\beta)\big|\leq \frac{3}{2}\epsilon_1nF(\beta_{\mathtt{anc}}). 
       \end{eqnarray}
       So we complete the proof. 
   
\end{proof}
To prove $\tilde F(\beta)$ is a qualified coreset, we need to extend Lemma~\ref{lemma2} to any $\beta \in \mathbb{B}(\beta_{\mathtt{anc}},R)$. 

For this purpose, we discretize the region $\mathbb{B}(\beta_{\mathtt{anc}},R)$ first (the discretization is only used for our analysis, and we do not need to build the grid in reality). 
Imagine that we build a uniform grid inside $\mathbb{B}(\beta_{\mathtt{anc}},R)$ with the side length being equal to $\frac{\epsilon_2R}{\sqrt d}$, where the exact value of $\epsilon_2$ is to be determined later. 
Inside each grid cell of $\mathbb{B}(\beta_{\mathtt{anc}},R)$, we pick an arbitrary point as its representative point and let $G$ be the set consisting of all the representative points. Based on the formula of the volume of a ball in $\mathbb{R}^d$, we have 
\begin{eqnarray}\label{eq:ball}
|G|=\left(O\left(\frac{2\sqrt{\pi e}}{\epsilon_2}\right)\right)^d.\label{for-grid}
\end{eqnarray}

So we can simply increase the sample size of Lemma~\ref{lemma2}, and take the union bound over all $\beta\in G$ so as to extend the result as follows.

\begin{lemma}
\label{lem-union}
Suppose $\epsilon_1\geq 0$. In the sample size (\ref{for-hoeffding-qsize}) of Lemma~\ref{lemma_hoeffding}, we set $\delta=\epsilon_12^{j-1}H$ for $j=0, 1, \cdots, N$, respectively, and replace $\lambda$ by $\frac{\lambda}{(N+1)|G|}$. The following
	\begin{eqnarray}
	\big|\tilde{F}(\beta) - F(\beta)\big| \leq \frac{3}{2}\epsilon_1 F(\beta_{\mathtt{anc}})
	\end{eqnarray}	
	holds for any $\beta \in G$, with probability at least $1-\lambda$. 
\end{lemma}

Following Lemma~\ref{lem-union}, we further derive a uniform bound over all $\beta\in\mathbb{B}(\beta_{\mathtt{anc}},R)$ (not just in $G$). 
For any $\beta\in \mathbb{B}(\beta_{\mathtt{anc}},R)$, we let $\beta'\in G$ be the representative point of the cell containing $\beta$.
	Then we have $\|\beta - \beta'\|\leq \epsilon_2R$. We define $M'\coloneqq\max_{1\leq i \leq n}\max\limits_{\beta\in\mathcal{B}(\beta_{\mathtt{anc}},R)}\|\nabla f(\beta,x_i,y_i)\|$. By Assumption~\ref{assump-lip} we immediately know $M'\leq M+LR$. 	
By using  the similar manner of~\eqref{eq:Lip-bound},  for any $1\leq i\leq n$ we have
	\begin{eqnarray}
	\big|f_i(\beta)-f_i(\beta')\big|\leq \epsilon_2M'R+\frac{1}{2}L\epsilon_2^2R^2.
	\end{eqnarray}

   This implies both
	\begin{eqnarray}
	|F(\beta)-F(\beta')| \text{ and } |\tilde F(\beta)-\tilde F(\beta')|\nonumber\\
	\leq	\epsilon_2M'R+\frac{1}{2}L\epsilon_2^2R^2. \label{for-sum-1}
	\end{eqnarray}

Using triangle inequality, we obtain
\begin{eqnarray}
&&|\tilde F(\beta)-F(\beta)| \nonumber\\
&\leq&|\tilde F(\beta)-\tilde F(\beta')|+|\tilde F(\beta')-F(\beta')|\nonumber\\
&&+|F(\beta')-F(\beta)|\nonumber\\
&\leq& \frac{3}{2}\epsilon_1 F(\beta_{\mathtt{anc}})+2\times (\epsilon_2M'R+\frac{1}{2}L\epsilon_2^2R^2),
\end{eqnarray}
where the last inequality follows from Lemma~\ref{lem-union} (note $\beta'\in G$) and (\ref{for-sum-1}).  By letting $\epsilon_1=\frac{2m\epsilon}{7F(\beta_{\mathtt{anc}})}$ and $\epsilon_2=\frac{2\epsilon_1F(\beta_{\mathtt{anc}})}{R\left(\sqrt{M'^2+2L\epsilon_1F(\beta_{\mathtt{anc}})}+M'\right)}$, we have $|\tilde F(\beta)-F(\beta)|\leq \epsilon F(\beta)$ via simple calculations. That is, the returned vector $W=[w_1, w_2, \cdots, w_n]$ is a qualified coreset  $\mathcal{CS}_\epsilon(\beta_{\mathtt{anc}}, R)$.

Last, it remains to specify the obtained coreset size. To guarantee the success probability to be  at least  $1-1/n$, we set $\lambda=1/n$. Then we can compute the coreset size, {\em i.e.,} the number of non-zero entries of $W$, 
which equals 
\begin{eqnarray}
\sum^N_{j=0}|Q_j|=\tilde{O}\left(\left(\frac{H+MR+LR^2}{m}\right)^2\cdot\frac{d}{\epsilon^2}\right) \label{for-ana-coresetsize}
\end{eqnarray}
 (by combining~\eqref{eq:coreset-size}, with the selection of $\delta$ in Lemma~\ref{lemma2}, the choice of $\lambda$ in Lemma~\ref{lem-union} along with~\eqref{eq:ball}, and the definition of $\epsilon_1$).  

For the case that $\beta$ is restricted to have at most $k$ non-zero entries ({\em i.e.,} sparse optimization) 
, we revisit the size $|G|$ in (\ref{for-grid}). For a $d$-dimensional vector, there are ${d\choose k}$ different combinations for the positions of the $k$ non-zero entries. Thus $\beta$ can be only located in the union of ${d\choose k}$ $k$-dimensional subspaces (similar idea was also used for analyzing compressed sensing~\cite{baraniuk2006johnson}). In other words, we just need to build the grid  (only for the sake of analysis) in the union of ${d\choose k}$ $k$-dimensional balls instead of the whole $\mathbb{B}(\beta_{\mathtt{anc}},R)$. Consequently, the new size $|G|$ is $\left({d\choose k}\cdot \left(O\left(\frac{2\sqrt{\pi e}}{\epsilon_2}\right)\right)^k\right)$, and the coreset size is reduced to 
$\tilde{O}\left(\left(\frac{H+MR+LR^2}{m}\right)^2\cdot\frac{k\log d}{\epsilon^2}\right)$.

\subsection{Gradient Preservation}
\label{sec-grad}
Besides the quality guarantee (\ref{for-loccoreset2}), our local coreset also enjoys another favorable property. In this section, we show that the gradient $\nabla\tilde F(\beta)$ can be approximately preserved as well, {\em i.e.,}    $\nabla\tilde F(\beta)\approx \nabla F(\beta)$ for any $\beta\in \mathbb{B}(\beta_{\mathtt{anc}},R)$. Because the trajectory of $\beta$ is guided by the gradients, this property gives a hint that our eventually obtained $\beta$ is likely to be close to the optimal hypothesis $\beta^*$ (we also validate this property in our experiments). In some scenarios like statistical inference and parameter estimation, we expect to achieve not only an  almost minimal loss $F(\beta)$, but also a small difference between  $\beta$ and $\beta^*$.

Given a vector $v\in\mathbb{R}^d$, we use $v_{[l]}$ to denote its $l$-th coordinate value, for $l=1, 2, \hdots, d$. Under Assumption \ref{assump-lip}, we obtain (similar with (\ref{eqa: betabound})), for any $1\leq i\leq n$, 
\begin{eqnarray}
\nabla f_i(\beta)_{[l]}\in \nabla f_i(\beta_{\mathtt{anc}})_{[l]}\pm LR.
\end{eqnarray} 
We can apply a similar line of analysis as in Section~\ref{sec-ana} to obtain Theorem~\ref{gradient_coreset}. We need the following modifications.   
First, we need to change the sample size $Q_j$ (and similarly the total coreset size in (\ref{for-ana-coresetsize})) of Algorithm~\ref{alg: local coreset} because we now consider a different objective. Also, we achieve an additive error for the gradient, instead of the $(1\pm \epsilon)$-multiplicative error as (\ref{for-loccoreset2}). The reason is that the gradient can be almost equal to $0$, if the solution approaches to a local or global optimum (but the objective value~(\ref{for-er}) is usually not equal to $0$, {\em e.g.,} we often add a non-zero penalty item to the objective function).

	\begin{theorem}
	\label{gradient_coreset}
Let $\sigma>0$ be any given small number.  With probability $1-\frac{1}{n}$, Algorithm \ref{alg: local coreset} can return a vector $W$ with $\tilde{O}\left(\frac{(LR+M)^2}{\sigma^2}\cdot d\right)$ non-zero entries, such that for any $ \beta \in \mathbb{B}(\beta_{\mathtt{anc}},R)$ and $1\leq l\leq d$, 
\begin{eqnarray}
\nabla \tilde F(\beta)_{[l]}\in \nabla F(\beta)_{[l]}\pm \sigma. 
\end{eqnarray}
 Furthermore, if the vector $\beta$ is restricted to have at most $k\in \mathbb{Z}^+$ non-zero entries in the hypothesis space $\mathbb{R}^d$, the number of non-zero entries of $W$ can be reduced to be $\tilde{O}\left(\frac{(LR+M)^2}{\sigma^2}\cdot k\log d \right)$.
	\end{theorem}
	\begin{remark}
	If we want to guarantee both  Theorem~\ref{local_coreset} and \ref{gradient_coreset}, we can just set the coreset size as the maximum over both cases.
	\end{remark}
\subsection{Beyond Gradient Descent}
\label{sec-beyond}
In Section~\ref{sec-ana}, our analysis relied on the fact that the function $f(\beta, x_i, y_i)$ is differentiable. However, for some ERM problems, the loss function can be non-differentiable. A representative example is the $l_1$-norm regularized regression, such as \cite{10.2307/2346178,DBLP:conf/aaai/LeeLAN06}. We consider the $l_p$ regularized regression with $0< p \leq 2$. Given a regularization parameter $\lambda>0$, the objective function can be written as 
\begin{eqnarray}
\label{eq:reg}
F(\beta)=\frac{1}{n}\sum^n_{i=1}g(\beta, x_i, y_i)+\lambda \|\beta\|_p, \label{for-er-beyond}
\end{eqnarray}
where the function $g(\beta, x_i, y_i)$ is assumed to be differentiable and satisfy Assumption~\ref{assump-lip}. We can easily cast (\ref{eq:reg}) to have the form of (\ref{for-er}) by setting $f(\beta, x_i, y_i)=g(\beta, x_i, y_i)+\lambda\|\beta\|_p$. Noting that a local $\epsilon$-coreset of the original problem obviously is also a local $\epsilon$-coreset of (\ref{eq:reg}), a coreset algorithm querying the values of all $g_i$s is capable to construct the coreset of (\ref{eq:reg}). Actually an $\epsilon$-coreset of (\ref{eq:reg}) can be constructed with only access to the values of all $f_i$s as well (See details in Appendix \ref{sec:beyond-app}).

\section{Sequential Coreset Framework and Applications}
\label{sec-app}
The local $\epsilon$-coreset constructed in Section~\ref{sec-construction} can be directly used for compressing input data. However, the trajectory of the hypothesis $\beta$ (although enjoying the locality property) may span a relatively large range globally in the space. As discussed in Remark~\ref{rem-the-quality}, the coreset size depends on the pre-specified local region range. Therefore, the coreset size can be high, if we want to build in one shot a local coreset  that covers the whole trajectory. 
This  motivates us to propose the \textbf{sequential coreset framework} (see Algorithm~\ref{alg: sequential algorithm}). 
  
In each round of Algorithm~\ref{alg: sequential algorithm}, we build the local coreset $\mathcal{CS}_\epsilon(\beta_t, R)$ and run the ``host'' algorithm $\mathcal{A}$ on it until either (\rmnum{1}) the result becomes stable inside $\mathbb{B}(\beta_t, R)$ or (\rmnum{2}) the hypothesis $\beta$ reaches the boundary of $\mathbb{B}(\beta_t, R)$\footnote{In practice, we can set a small number $\sigma\in (0,1)$ and deduce that the boundary is reached when $\|\beta_t-\beta\|>(1-\sigma)R$.}.  
For (\rmnum{1}), we just terminate the algorithm and output the result; for (\rmnum{2}), we update $\beta_t$ and proceed the next iteration.  
\begin{algorithm}[tb]
	\caption{Sequential Coreset Framework}
	\label{alg: sequential algorithm}
	\begin{algorithmic}
		\STATE {\bfseries Input:} An instance $P=\{(x_1, y_1),$ $ (x_2, y_2), \cdots, $ $(x_n, y_n)\}$ of the ERM problem (\ref{for-er}) with the initial solution $\beta_0$ and range $R>0$, an available gradient descent algorithm $\mathcal{A}$ as the ``host'', and the parameter $\epsilon\in(0,1)$. 
 		\begin{enumerate}
		\item For $t=0,1, \hdots$, build the local coreset $\mathcal{CS}_\epsilon(\beta_t, R)$ and run the host algorithm $\mathcal{A}$ on it until:
		\begin{enumerate}
		\item if the result becomes stable inside $\mathbb{B}(\beta_t, R)$, terminate the loop and return the current $\beta$;
		\item else, the current $\beta$ reaches the  boundary of $\mathbb{B}(\beta_t, R)$, and then set $\beta_{t+1}=\beta$ and $t=t+1$. 
		\end{enumerate}
		\end{enumerate}

	\end{algorithmic}
\end{algorithm}
Following the sequential coreset framework, we consider its applications for several ERM problems in machine learning.

\textbf{Ridge regression.} In the original linear regression problem, the data space $\mathbb{X}=\mathbb{R}^d$ and the response space $\mathbb{Y}=\mathbb{R}$, and the goal is to find a vector $\beta\in \mathbb{R}^d$ such that the objective function $F(\beta)=\frac{1}{n}\sum^n_{i=1}|\langle x_i,\beta \rangle-y_i|^2$ is minimized. For Ridge regression~\cite{ridge}, we add a squared $l_2$-norm penalty and the objective function becomes 
\begin{eqnarray}
F(\beta)=\frac{1}{n}\sum^n_{i=1}|\langle x_i,\beta \rangle-y_i|^2+\lambda \|\beta\|^2_2, \label{for-ridge}
\end{eqnarray}
where $\lambda>0$ is a regularization parameter. Consequently, the loss function $f(\beta, x_i, y_i)$ of (\ref{for-ridge}) is taken as $|\langle x_i,\beta \rangle-y_i|^2+\lambda\|\beta\|^2_2$.

\textbf{Lasso regression.} Another popular regularized regression model is Lasso~\cite{10.2307/2346178}.  Compared to (\ref{for-ridge}), the only difference is that we use an $l_1$-norm penalty {\em i.e., }
\begin{eqnarray}
F(\beta)=\frac{1}{n}\sum^n_{i=1}|\langle x_i,\beta \rangle-y_i|^2+\lambda \|\beta\|_1, \label{for-lasso}
\end{eqnarray}
where $\lambda>0$ is a regularization parameter. The loss function $f(\beta, x_i, y_i)$ of (\ref{for-lasso}) is $|\langle x_i,\beta \rangle-y_i|^2+\lambda\|\beta\|_1$. A key advantage of Lasso is that the returned $\beta$ is a sparse vector. 
The objective function (\ref{for-lasso}) is not differentiable, but it can still be solved by our sequential coreset framework as discussed in Section~\ref{sec-beyond}. 

\textbf{Logistic regression.} For Logistic regression, the response is binary, {\em i.e.,} $y_i=0$ or $1$~\cite{CRAMER2004613}. The objective function 
\begin{eqnarray}
F(\beta)=-\frac{1}{n}\sum^n_{i=1}\Big\{y_i\log g(\langle x_i,\beta\rangle)+\nonumber\\
(1-y_i)\log \big(1-g(\langle x_i,\beta\rangle)\big)\Big\}, \label{for-log}
\end{eqnarray}
where $g(t):=\frac{1}{1+e^{-t}}$ (the logistic function). We may add an $l_1$ or $l_2$-norm penalty to (\ref{for-log}), in the same way as (\ref{for-ridge}) and (\ref{for-lasso}). The loss function $f(\beta, x_i, y_i)$ for Logistic regression is $-y_i\log g(\langle x_i,\beta\rangle)-(1-y_i)\log \big(1-g(\langle x_i,\beta\rangle)\big)$. 

\textbf{Gaussian Mixture Model (GMM).} As emphasized before, our local coreset method does not require the objective function to be convex. Here,  we consider a typical non-convex example: GMM training~\cite{bishop2006pattern}. A mixture of $k$ Gaussian kernels is represented with $\beta \coloneqq [(\omega_1, \mu_1, \Sigma_1), \hdots, (\omega_k, \mu_k, \Sigma_k)]$, where $\omega_1, \hdots, \omega_k\geq 0$, $\sum^k_{j=1}\omega_j=1$, and each $(\mu_j, \Sigma_j)$ is the mean and covariance matrix of the $j$-th Gaussian in $\mathbb{R}^D$. GMM is an unsupervised learning problem, where the training dataset contains $\{x_1, \cdots, x_n\}\subset\mathbb{R}^D$, and the goal is to minimize the objective function
\begin{eqnarray}
F(\beta)=-\frac{1}{n}\sum^n_{i=1}\log\Big(\sum^k_{j=1}\omega_j N(x_i, \mu_j, \Sigma_j)\Big), \label{for-gmm}
\end{eqnarray}
where $N(x_i, \mu_j, \Sigma_j)=\frac{1}{\sqrt{(2\pi)^D|\Sigma_j|}}\exp(-\frac{1}{2}(x_i-\mu_j)^T\Sigma^{-1}_j(x_i-\mu_j))$; so $f(x_i, \beta)=-\log\Big(\sum^k_{j=1}\omega_j N(x_i, \mu_j, \Sigma_j)\Big)$ for (\ref{for-gmm}). It is worth noting that (\ref{for-gmm}) is differentiable and Lipschitz smooth and thus can be solved via the gradient descent method. However, the expectation-maximization (EM) method is more popular due to its simplicity and efficiency for GMM training.
Moreover, the EM method also has the locality property in practice.

In our experiment, we still use Algorithm~\ref{alg: sequential algorithm} to generate the sequential coreset, but run the EM algorithm as the ``host'' algorithm $\mathcal{A}$. 

\newcounter{sd3}
\begin{figure*}[htbp]
	\begin{center}
		\centerline
		{\includegraphics[width=0.5\columnwidth]{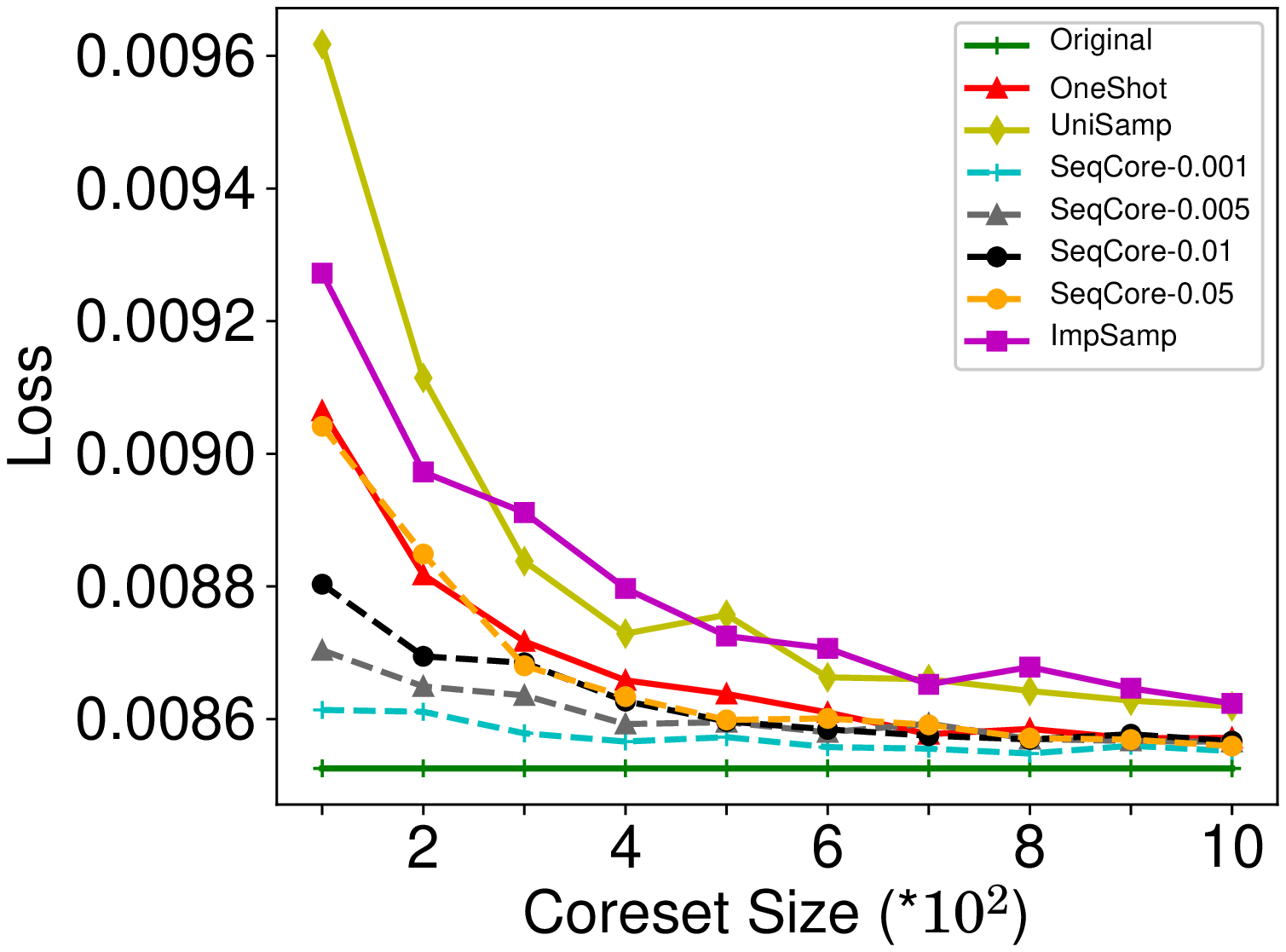} 
		\hspace{0.1in}
			\includegraphics[width=0.5\columnwidth]{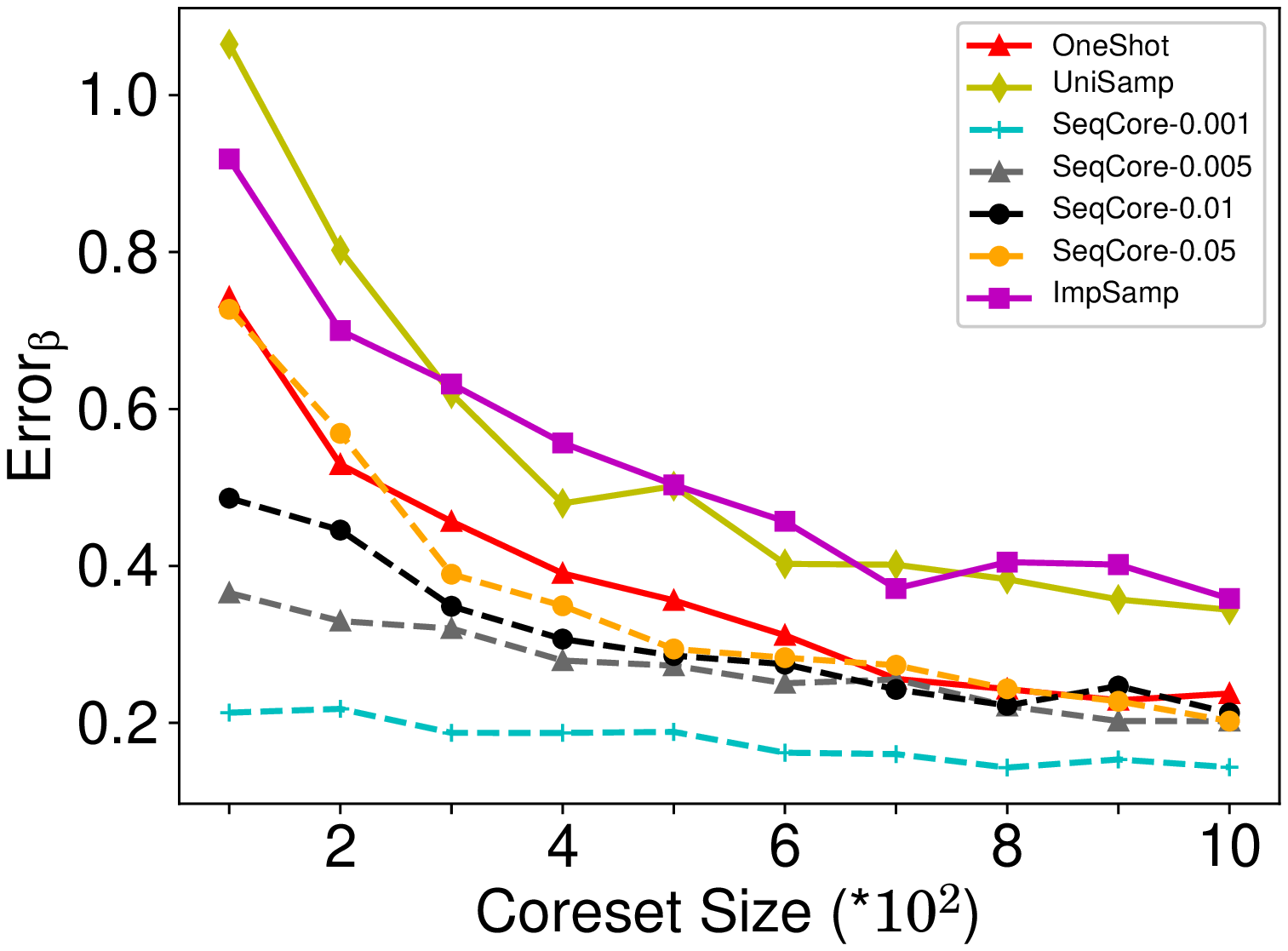}
					\hspace{0.1in}
			\includegraphics[width=0.5\columnwidth]{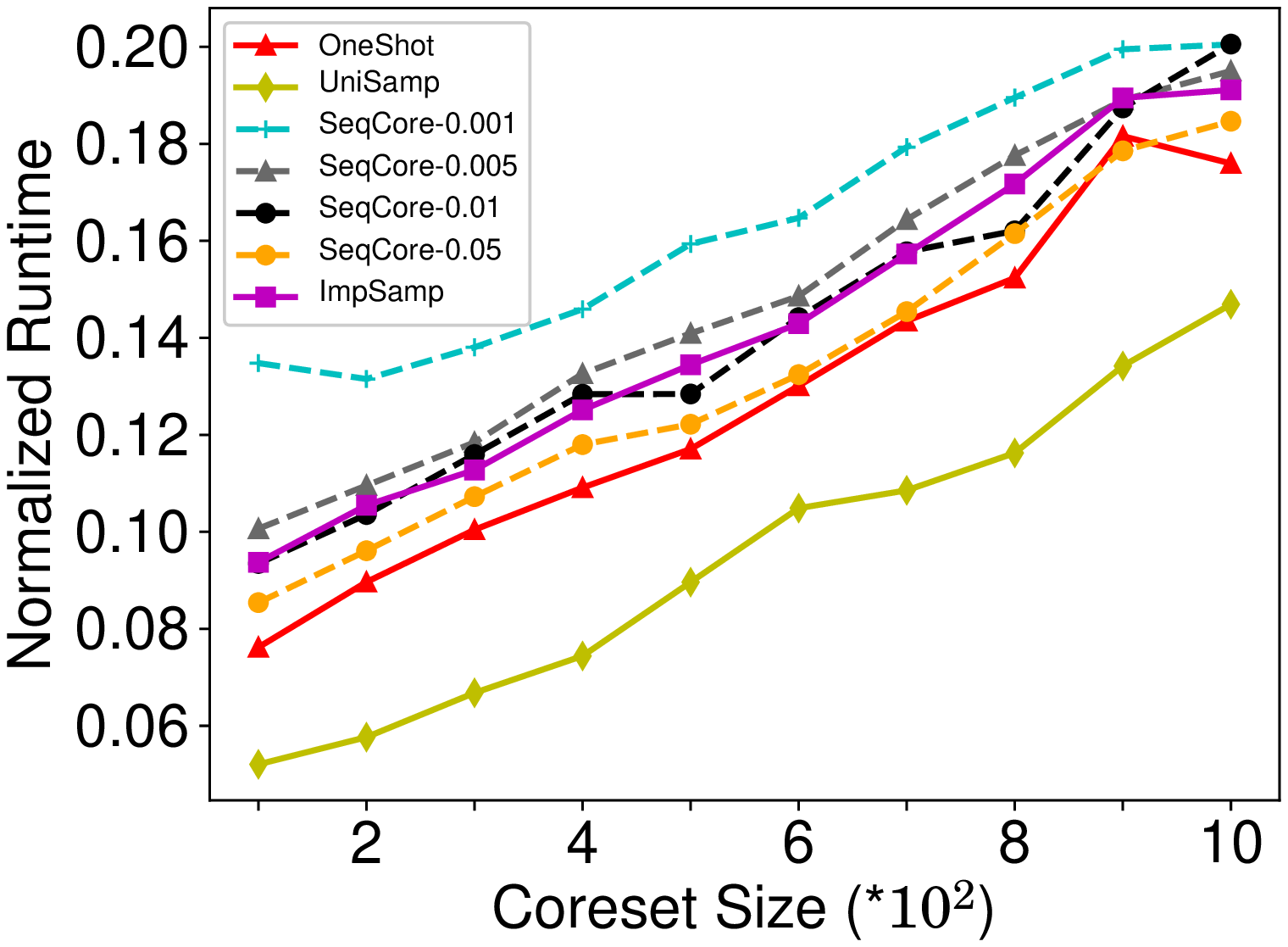}}
	\end{center}
	\vspace{-0.36in}
	\caption{The experimental results on \textsc{Appliances Energy} for Ridge regression ($\lambda=0.01$). } 
	\label{fig:Ridge_Appenergy}  
\end{figure*}

\begin{figure*}[htbp]
	\begin{center}
		\centerline
		{\includegraphics[width=0.5\columnwidth]{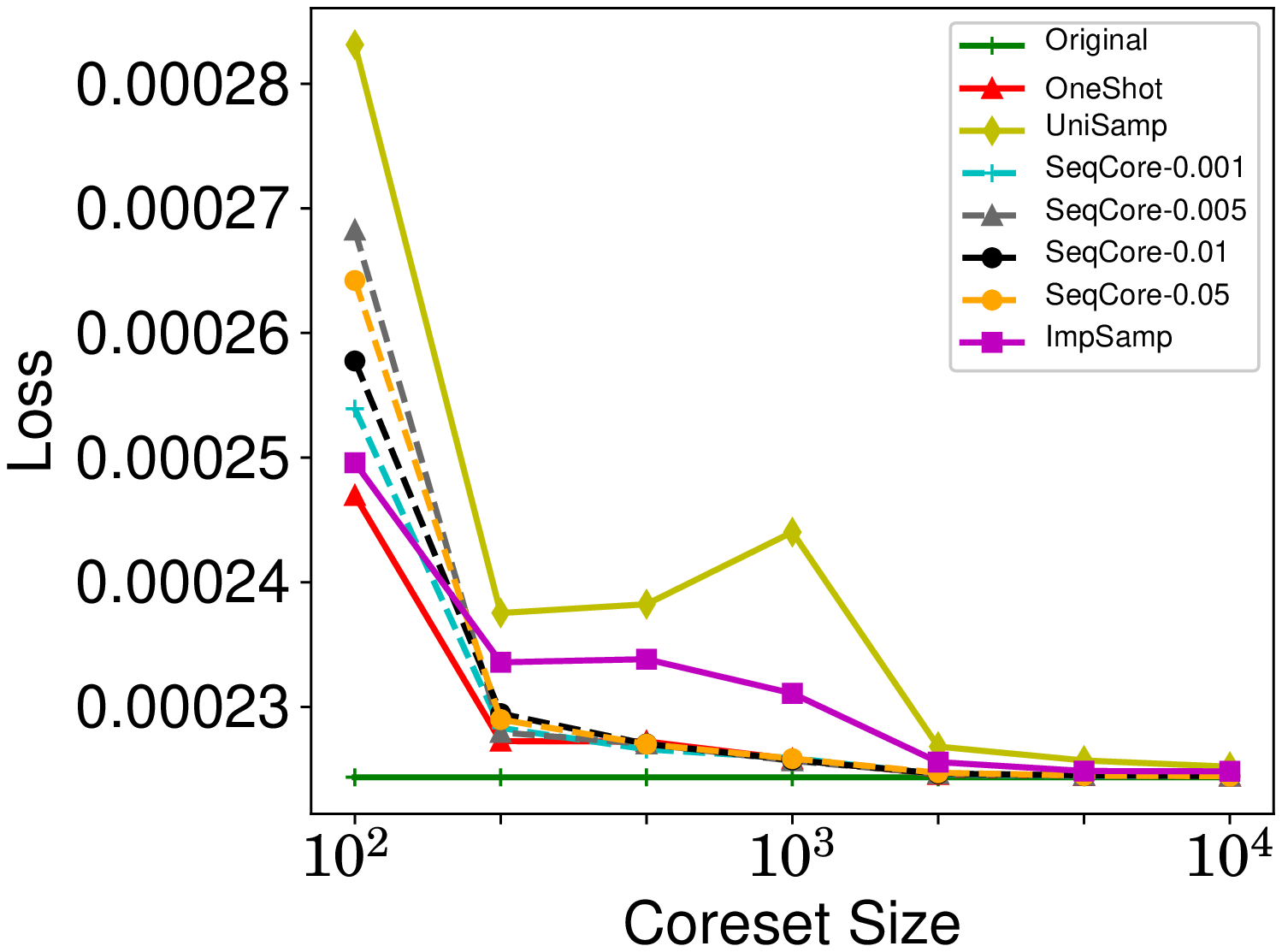} 
				\hspace{0.1in}
			\includegraphics[width=0.5\columnwidth]{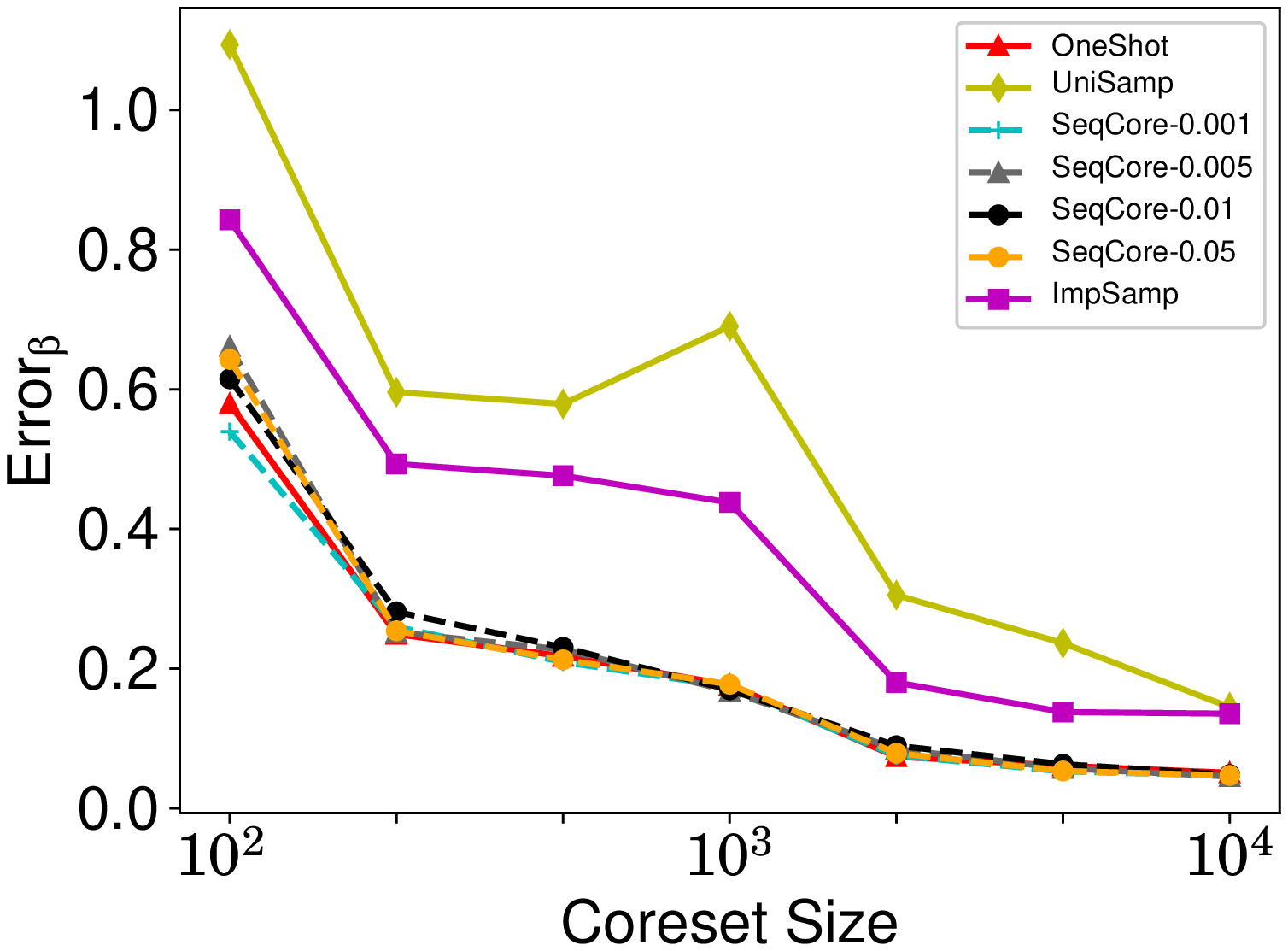}
					\hspace{0.1in}
			\includegraphics[width=0.5\columnwidth]{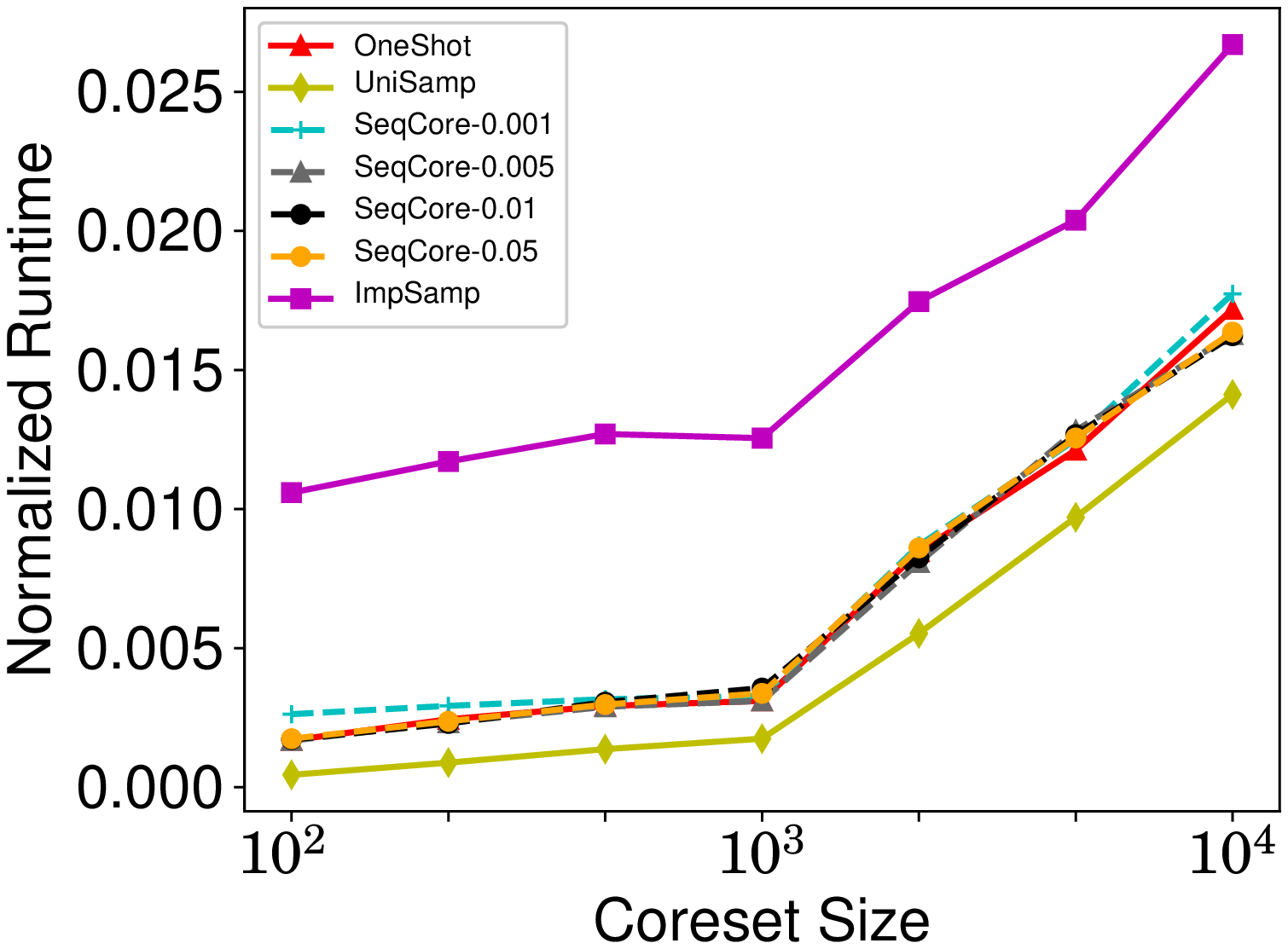}}
	\end{center}
	\vspace{-0.36in}
	\caption{The experimental results on \textsc{Facebook Comment} for Ridge regression ($\lambda=0.01$).} 
	\label{fig:Ridge_fb_comment}  
\end{figure*}

\begin{figure*}[htbp]
	\begin{center}
		\centerline
		{\includegraphics[width=0.5\columnwidth]{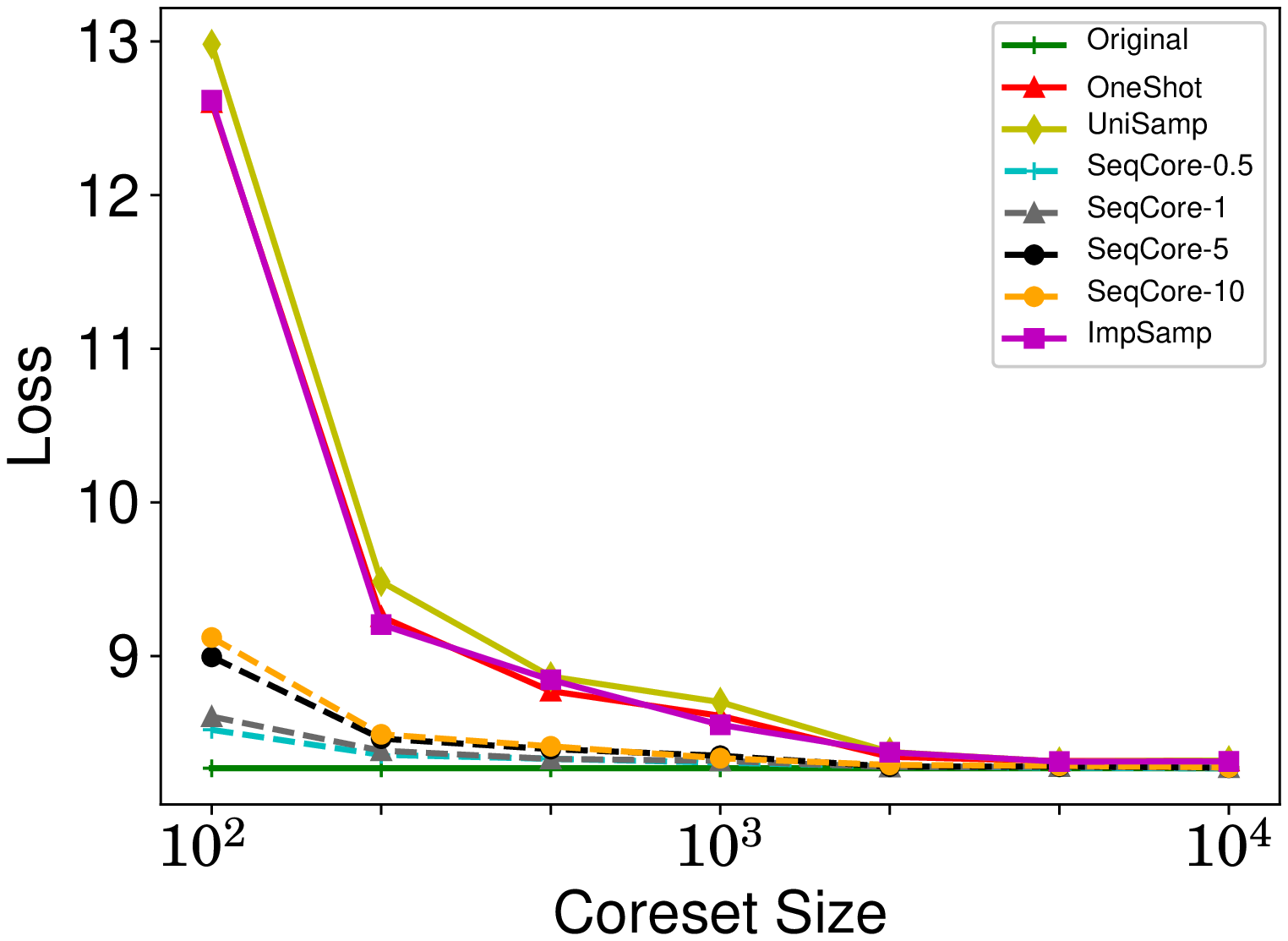} 
				\hspace{0.1in}
				\includegraphics[width=0.5\columnwidth]{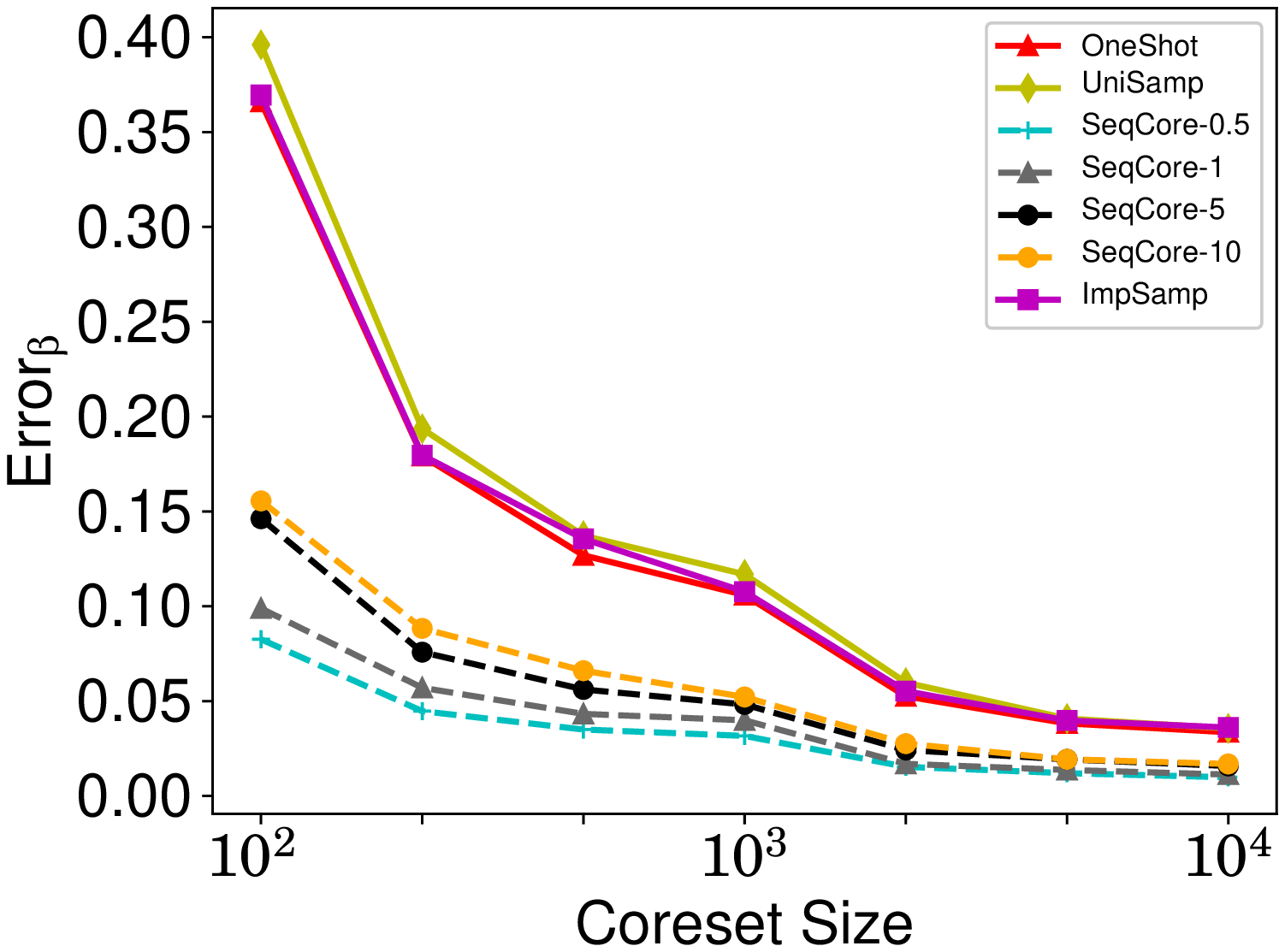}
						\hspace{0.1in}
				\includegraphics[width=0.5\columnwidth]{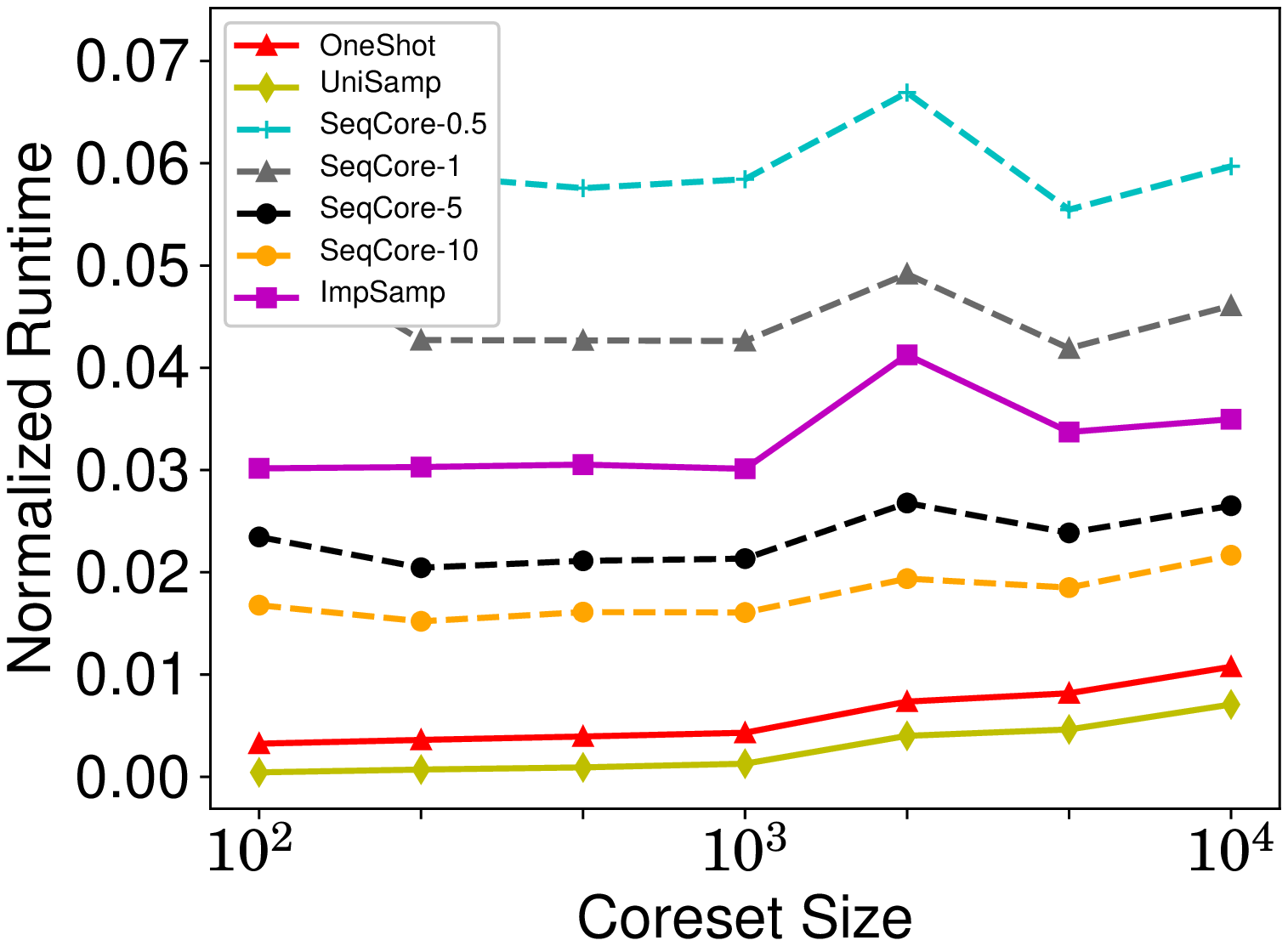}}
	\end{center}
	\vspace{-0.3in}
	\caption{The experimental results on the synthetic dataset for Ridge regression ($\lambda=0.01$)}
	\label{fig:Ridge_Syn}  
\end{figure*}

\begin{figure*}[htbp]
	\begin{center}
		\centerline
		{\includegraphics[width=0.5\columnwidth]{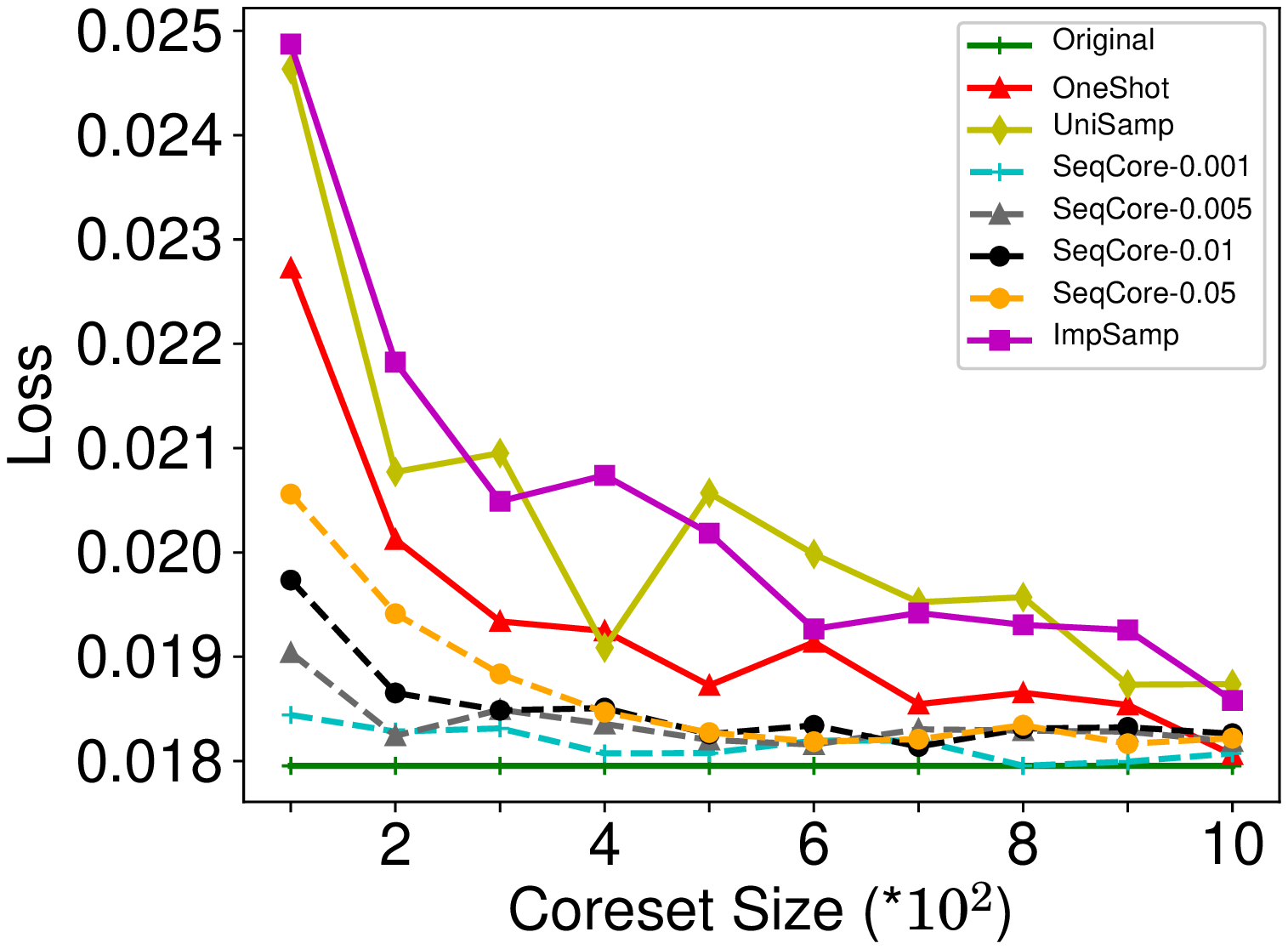}
				\hspace{0.1in}
			\includegraphics[width=0.5\columnwidth]{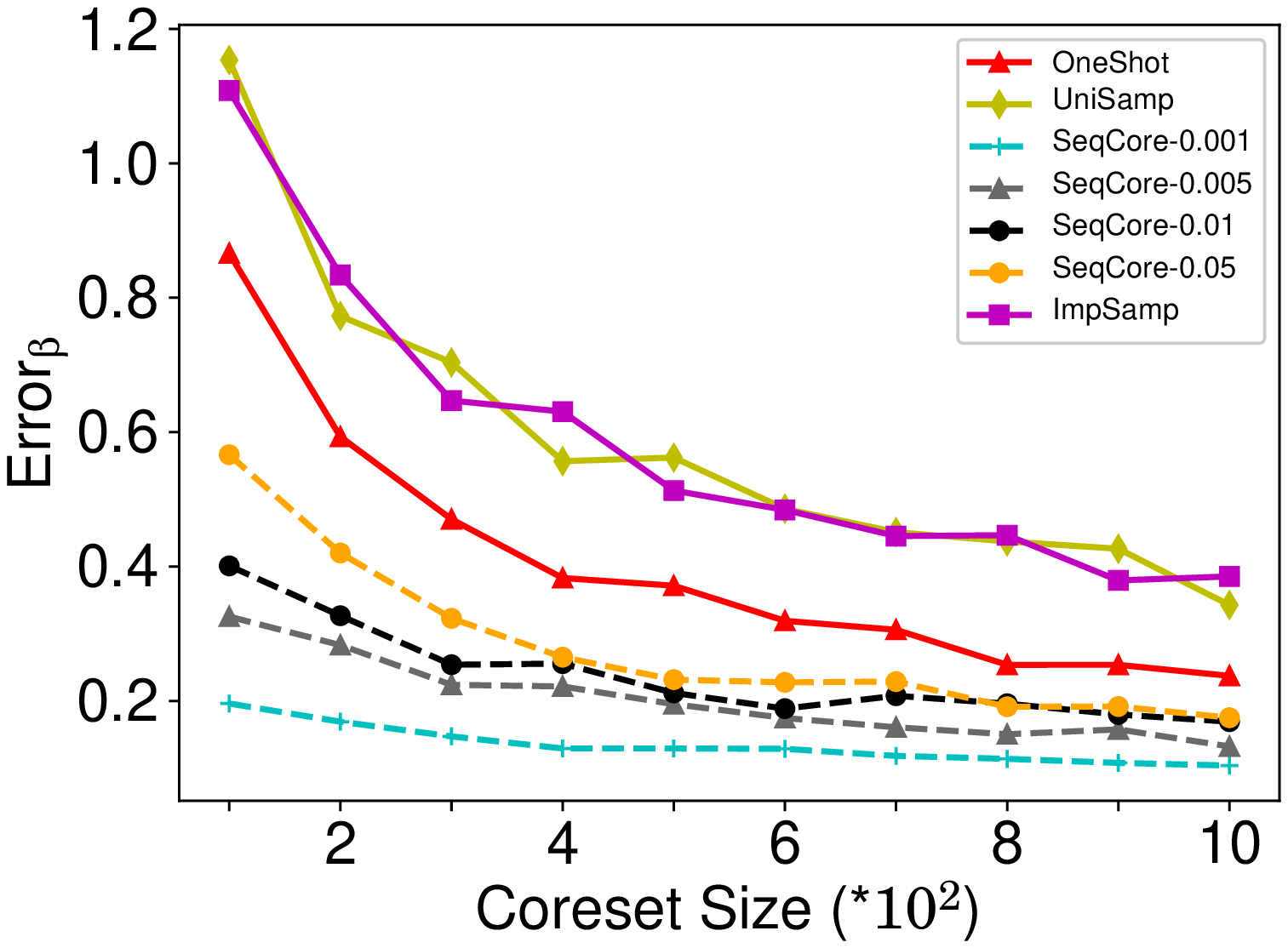}
					\hspace{0.1in}
			\includegraphics[width=0.5\columnwidth]{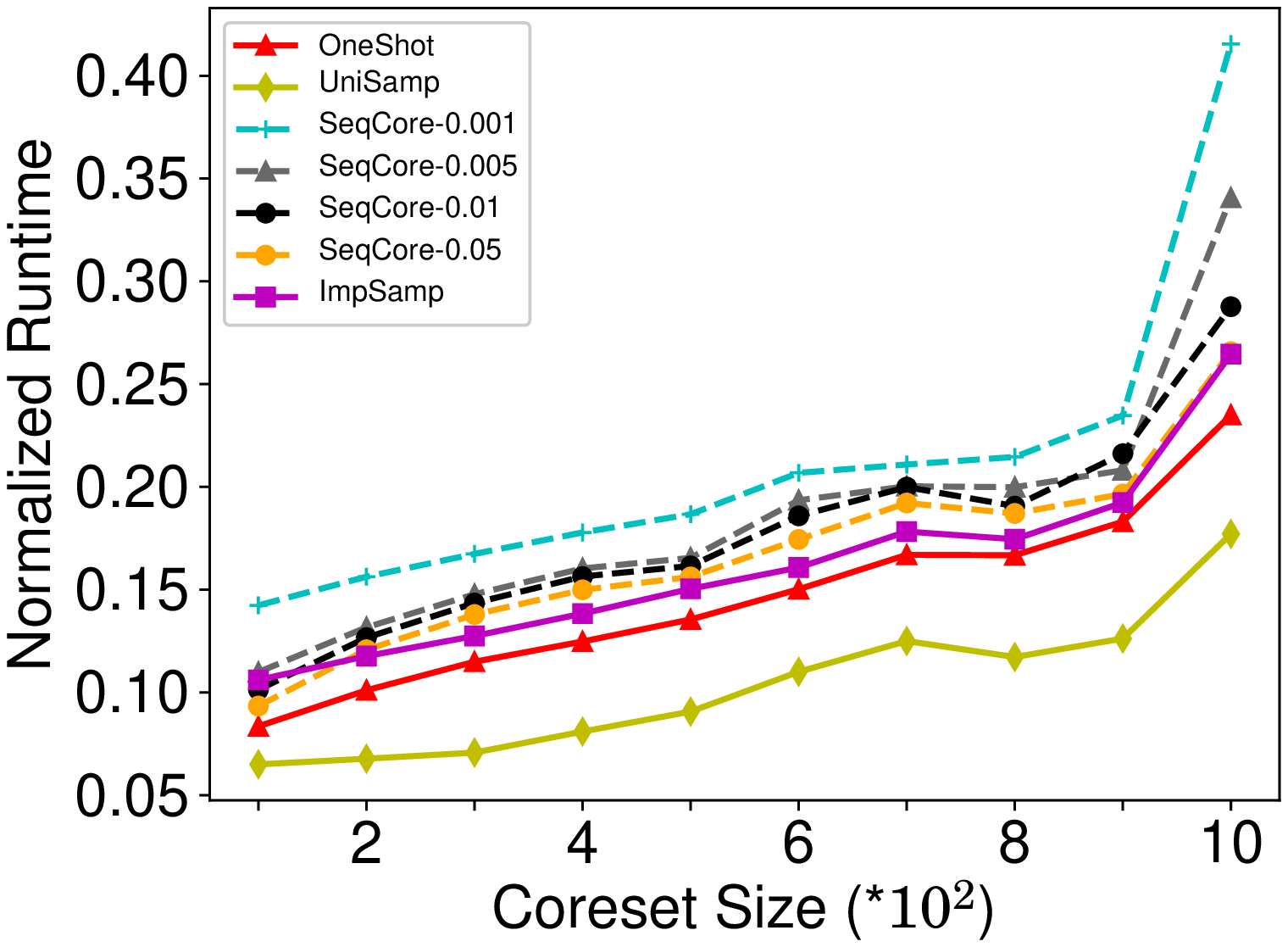}}
	\end{center}
	\vspace{-0.36in}
	\caption{The experimental results on \textsc{Appliances Energy} for Lasso regression ($\lambda=0.01$). } 
	\label{fig:Lasso_Appenergy}  
\end{figure*}

\begin{figure*}[htbp]
	\begin{center}
		\centerline
		{\includegraphics[width=0.5\columnwidth]{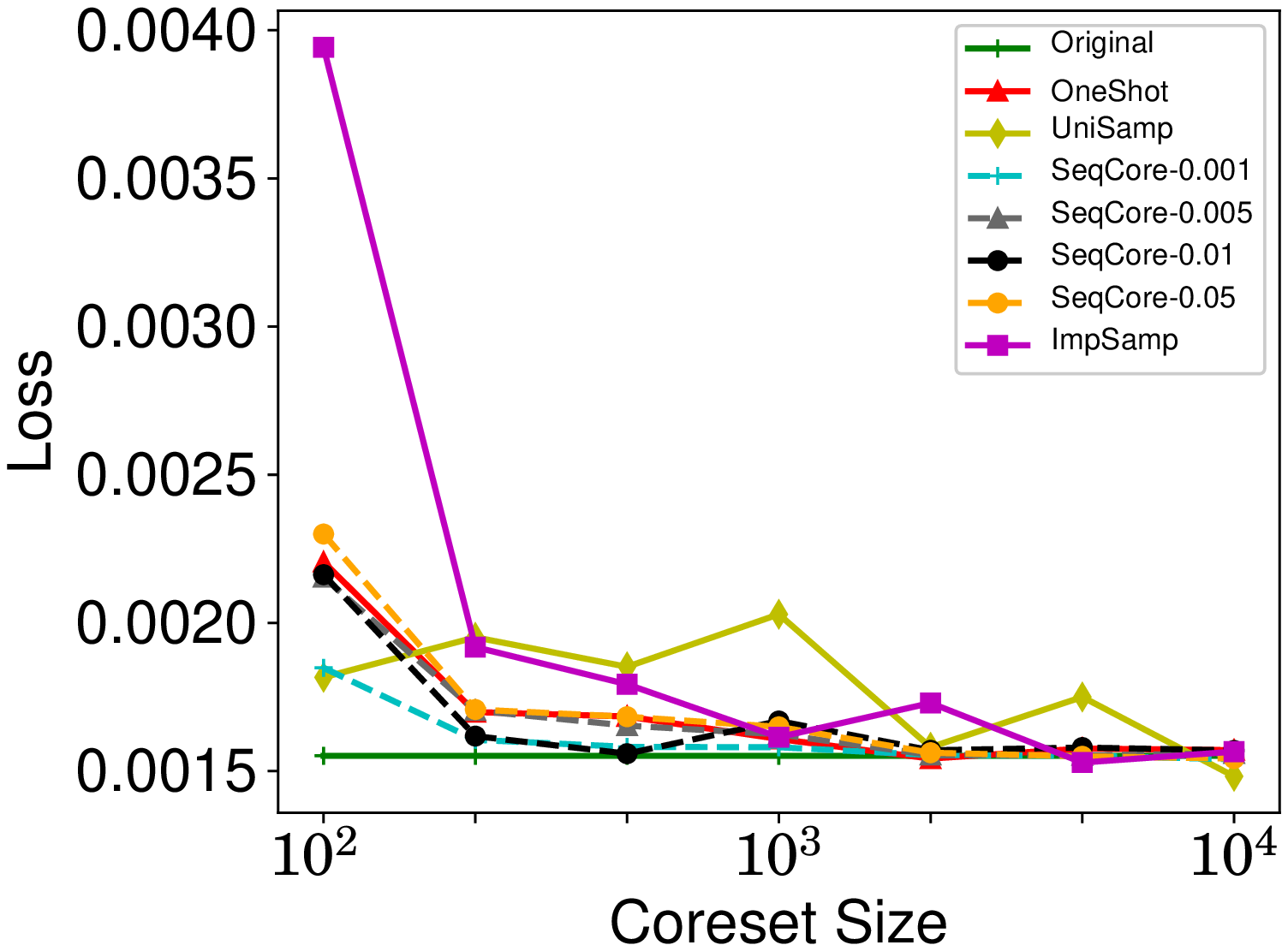} 
				\hspace{0.1in}
			\includegraphics[width=0.5\columnwidth]{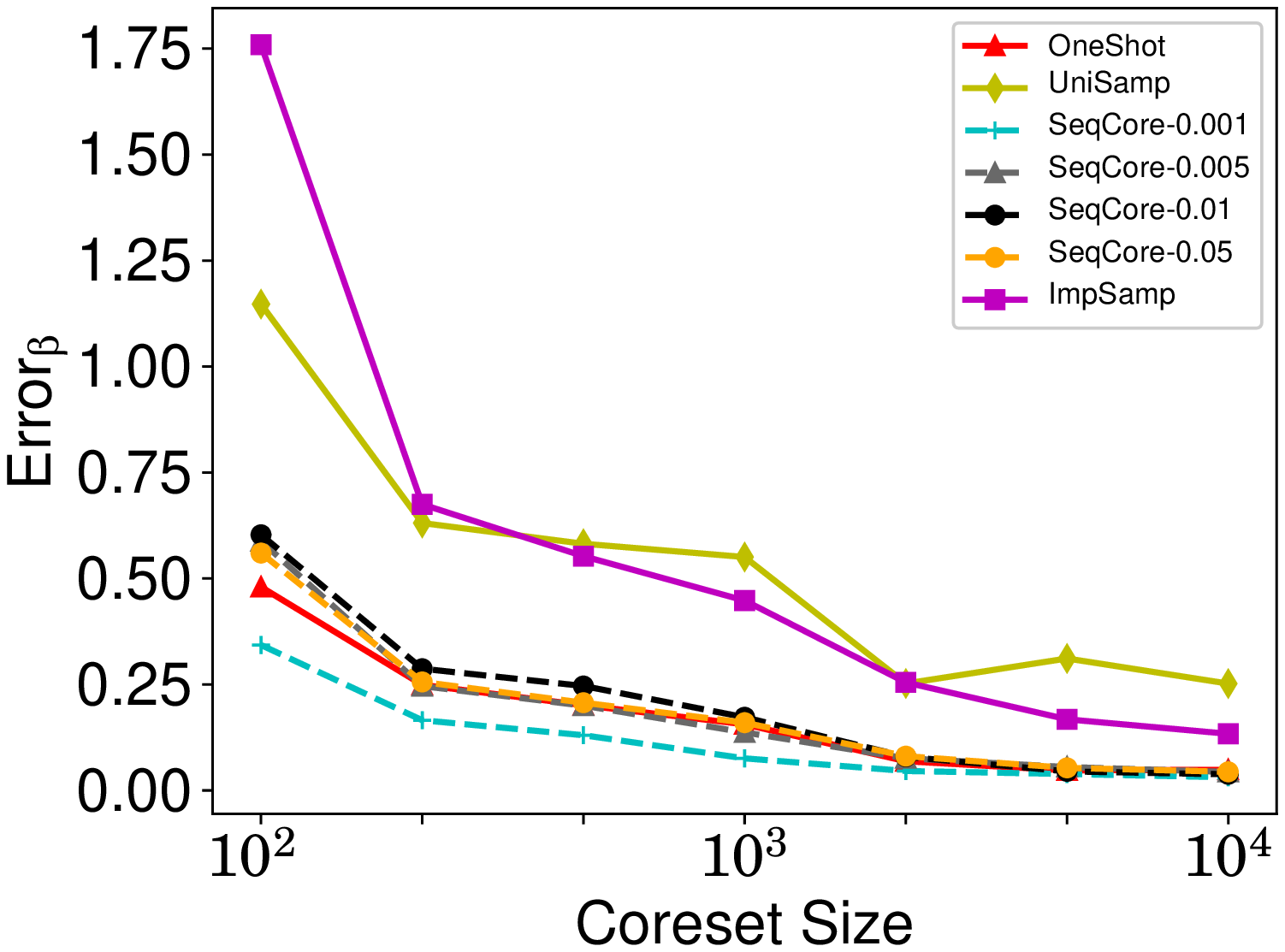}
					\hspace{0.1in}
			\includegraphics[width=0.5\columnwidth]{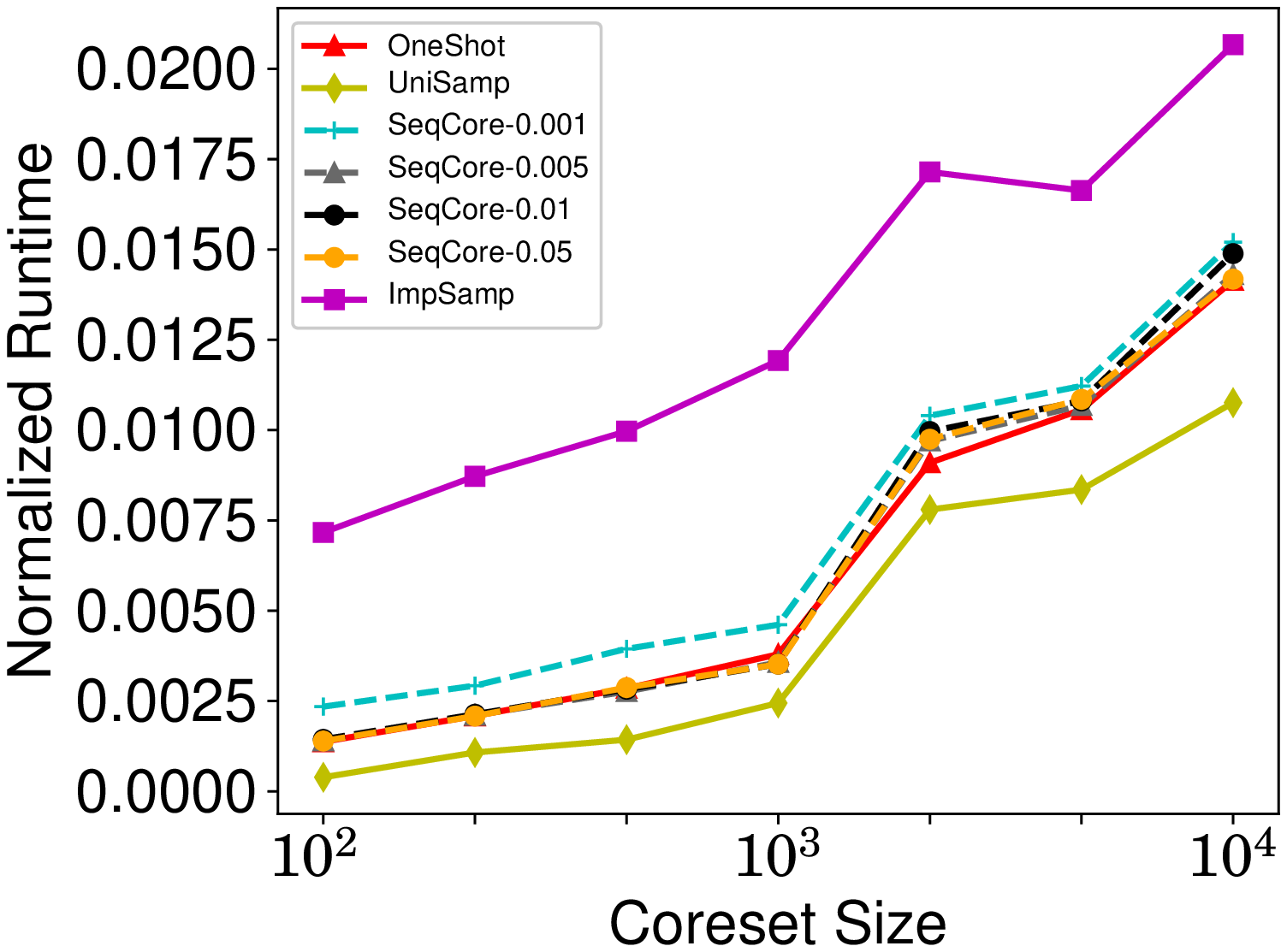}}
	\end{center}
	\vspace{-0.36in}
	\caption{The experimental results on \textsc{Facebook Comment} for Lasso regression ($\lambda=0.01$). } 
	\label{fig:Lasso_fb_comment}  
\end{figure*}

\section{Experimental Evaluation}
\label{sec-exp}
We evaluate the performance of our sequential coreset method for the applications mentioned in Section~\ref{sec-app}. All results were obtained on a server equipped with 2.4GHz Intel CPUs and 256GB main memory; the algorithms were implemented in Python.

\subsection{Ridge and Lasso Regression}
\label{sec-exp01}
We consider  Ridge and Lasso regression  first. 

\textbf{Datasets.} \textsc{Appliances Energy} is a dataset for predicting energy consumption which contains $19735$ points in $\mathbb{R}^{29}$~\cite{candanedo2017data}. \textsc{Facebook Comment} is a dataset for predicting comment which contains $602813$ points in $\mathbb{R}^{54}$~\cite{Sing1503:Comment}.
Furthermore, we generate a synthetic dataset of $10^6$ points in $\mathbb{R}^{50}$; each point is randomly sampled from the linear equation $y=\langle h, x\rangle$, where each coefficient of $h$ is sampled from $[-5, 5]$ uniformly at random; for each data point we also add a Gaussian noise $\mathcal{N}(0,4)$  to $y$.

\textbf{Compared methods.} As the host algorithm $\mathcal{A}$ in Algorithm~\ref{alg: sequential algorithm}, we apply the standard gradient descent algorithm. Fixing a coreset size, we consider several different data compression methods for comparison. (1) \textsc{Original}: directly run $\mathcal{A}$ on the original input data; (2) \textsc{UniSamp}: the simple uniform sampling; (3) \textsc{ImpSamp}: the importance sampling method~\cite{DBLP:conf/nips/TukanMF20}; 

(4) \textsc{SeqCore-$R$}: our sequential coreset method with a specified region range $R$;  

(5) \textsc{OneShot}: build the local coreset as Algorithm~\ref{alg: local coreset} in one-shot (without using the sequential idea)\footnote{For \textsc{OneShot}, we do not need to  specify the range $R$, if we fix the coreset size. The range is only used for our sequential coreset method because we need to re-build the coreset when $\beta$ reaches the boundary.}. 

\textbf{Results.}  
We consider three metrics to measure the performance: (1) the total loss, (2) the normalized error to the optimal $\beta^*$ (let $\mathtt{Error}_{\beta}=\frac{\Vert \beta-\beta^* \Vert_2}{\Vert\beta^*\Vert_2}$ where $\beta$ is the obtained solution and $\beta^*$ is the optimal solution obtained from \textsc{Original}), and (3) the normalized runtime (over the runtime of \textsc{Original}).  The results of Ridge regression are shown in Figures \ref{fig:Ridge_Appenergy}, \ref{fig:Ridge_fb_comment} and \ref{fig:Ridge_Syn} (averaged across $10$ trials). We can see that in general our proposed sequential coreset method has better performance on the loss and $\mathtt{Error}_{\beta}$, though sometimes it is slightly slower than \textsc{ImpSamp} if we set $R$ to be too small. \textsc{UniSamp} is always the fastest one (because it is just simple uniform sampling), but at the cost of inferior performance in total loss and model estimate error. \textsc{OneShot} is faster than \textsc{SeqCore-$R$} but often has worse loss and error. 
Similar results of Lasso regression are shown in Figure~\ref{fig:Lasso_Appenergy} and \ref{fig:Lasso_fb_comment}. 
Due to the space limit, more detailed experimental results (including the results on Logistic regression and GMM) are shown in the appendix.

\subsection{Gaussian Mixture Models}
\label{sec-exp02}
We directly generate the datasets by using the software package~\cite{scikit-learn} (the number of the data points $n = 10^5$). The host EM algorithm implementation is also from ~\cite{scikit-learn}. We separately vary the dimension, Gaussian Components number and coreset size. The experimental results are shown in Figure~\ref{fig:GMM_synth_time} and Figure~\ref{fig:GMM_synth_purity}. The \textit{purity} evaluates the similarity between our obtained clustering result and the ground truth~\cite{DBLP:books/daglib/0021593}. We can see our proposed sequential coreset method is slightly slower than \textsc{UniSamp} and \textsc{ImpSamp}~\cite{JMLR:v18:15-506}, but can achieve better purity.

\begin{figure*}[htbp]
	\begin{center}
		\centerline
		{\includegraphics[width=0.5\columnwidth]{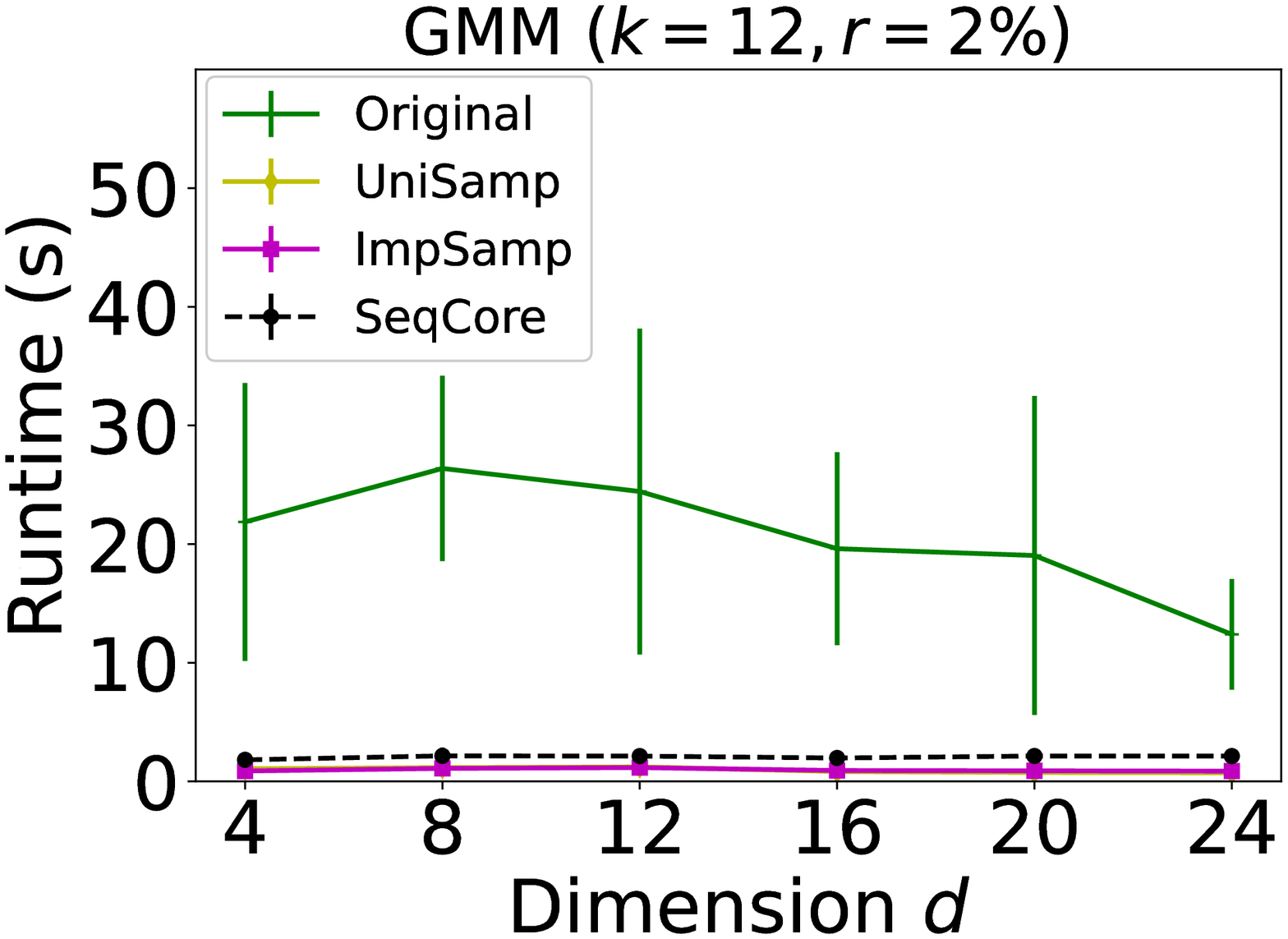} 
				\hspace{0.1in}
			\includegraphics[width=0.5\columnwidth]{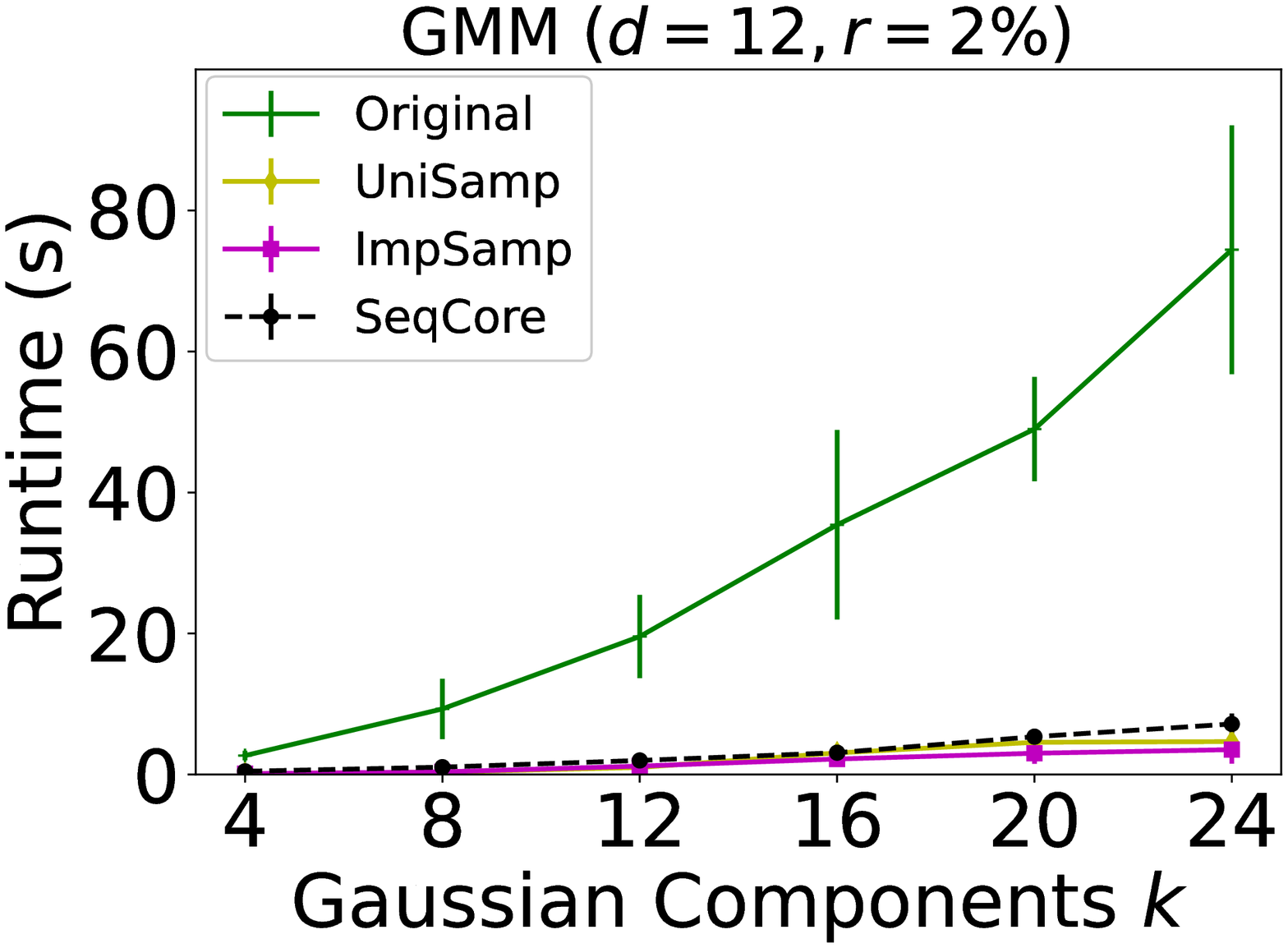}
					\hspace{0.1in}
			\includegraphics[width=0.5\columnwidth]{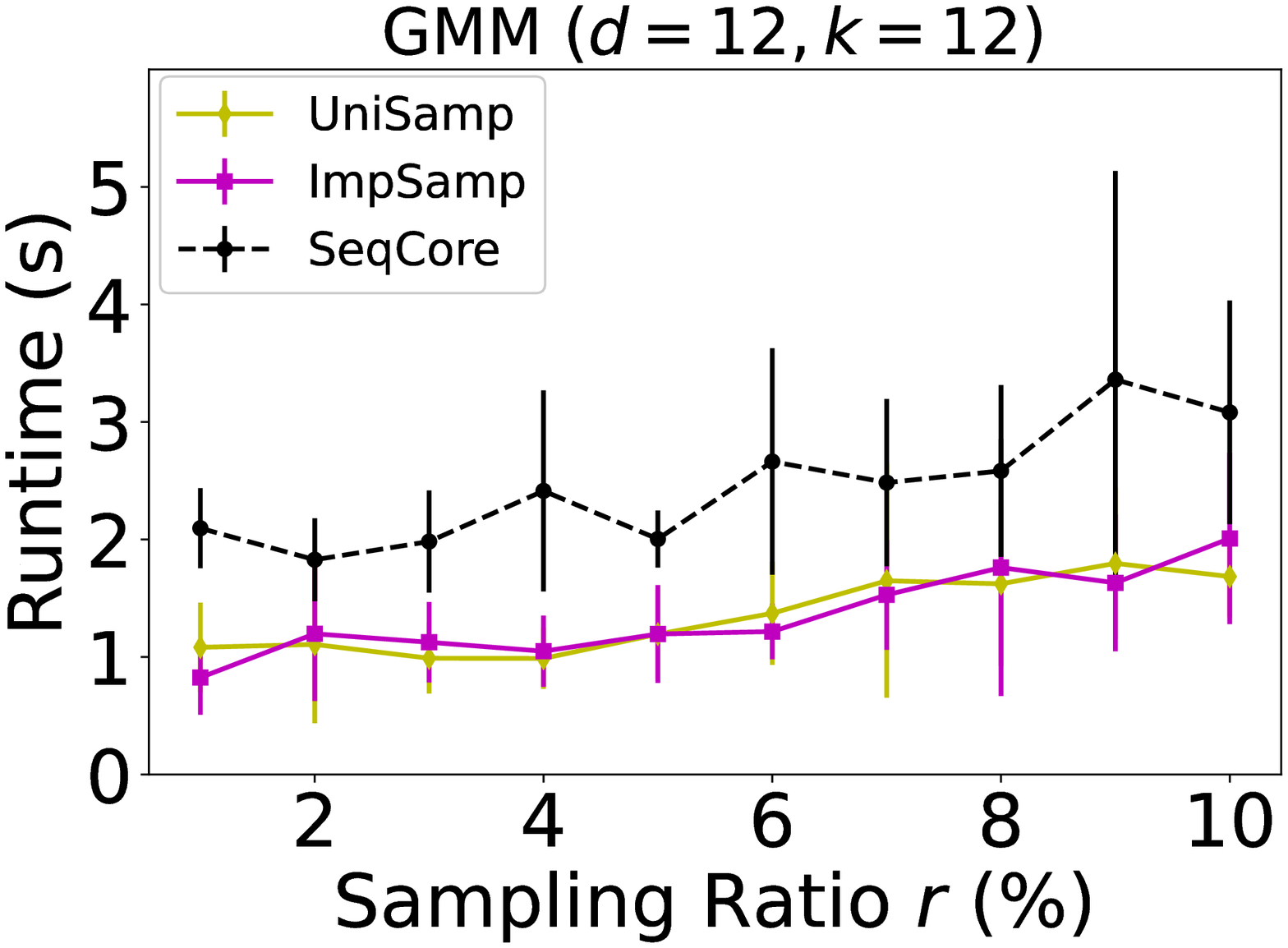}}
	\end{center}
	\vspace{-0.36in}
	\caption{The running times on \textsc{synthetic dataset} for Gaussian Mixture Models. } 
	\label{fig:GMM_synth_time}  
\end{figure*}

\begin{figure*}[htbp]
	\begin{center}
		\centerline
		{\includegraphics[width=0.5\columnwidth]{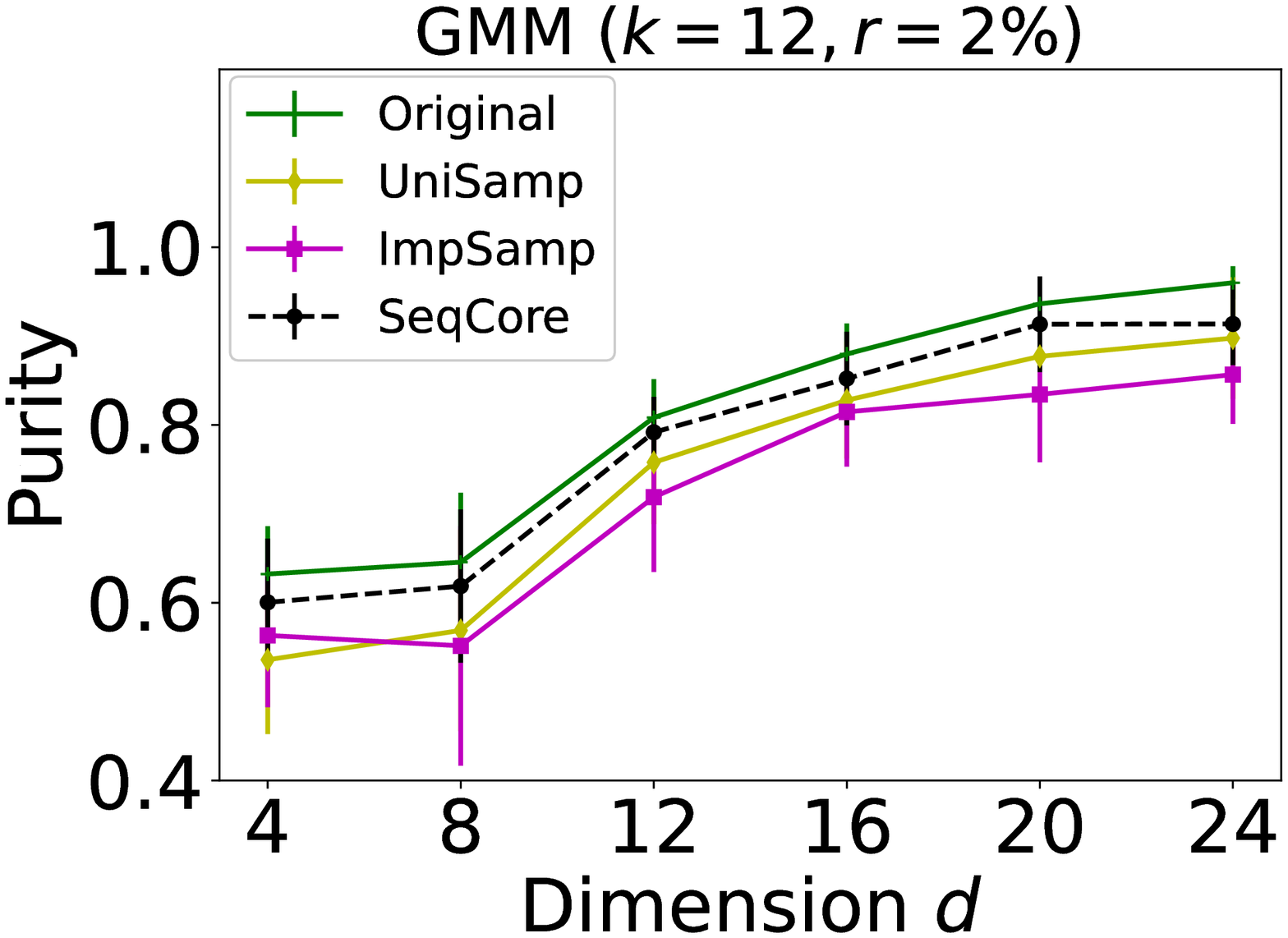} 
				\hspace{0.1in}
			\includegraphics[width=0.5\columnwidth]{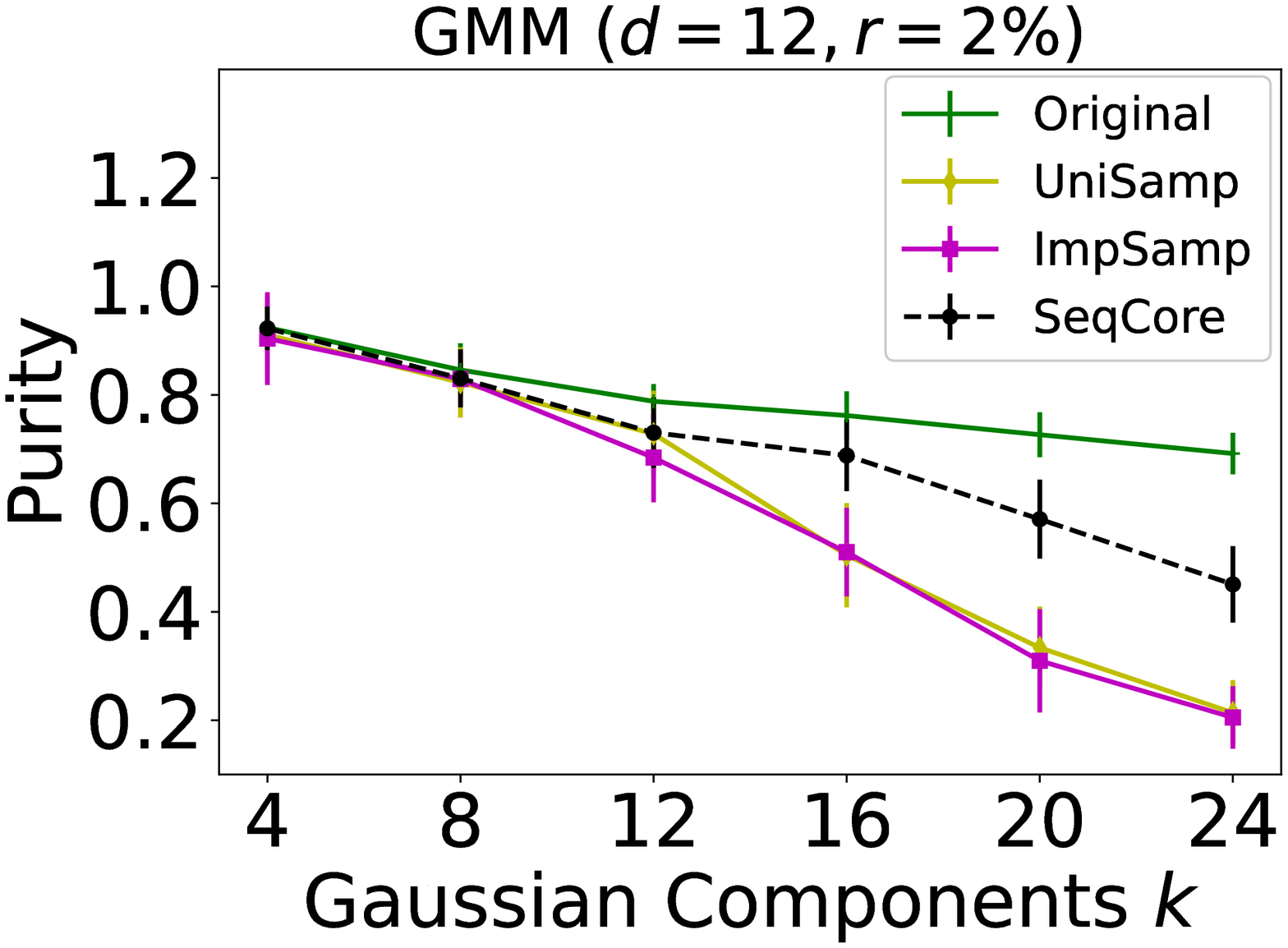}
					\hspace{0.1in}
			\includegraphics[width=0.5\columnwidth]{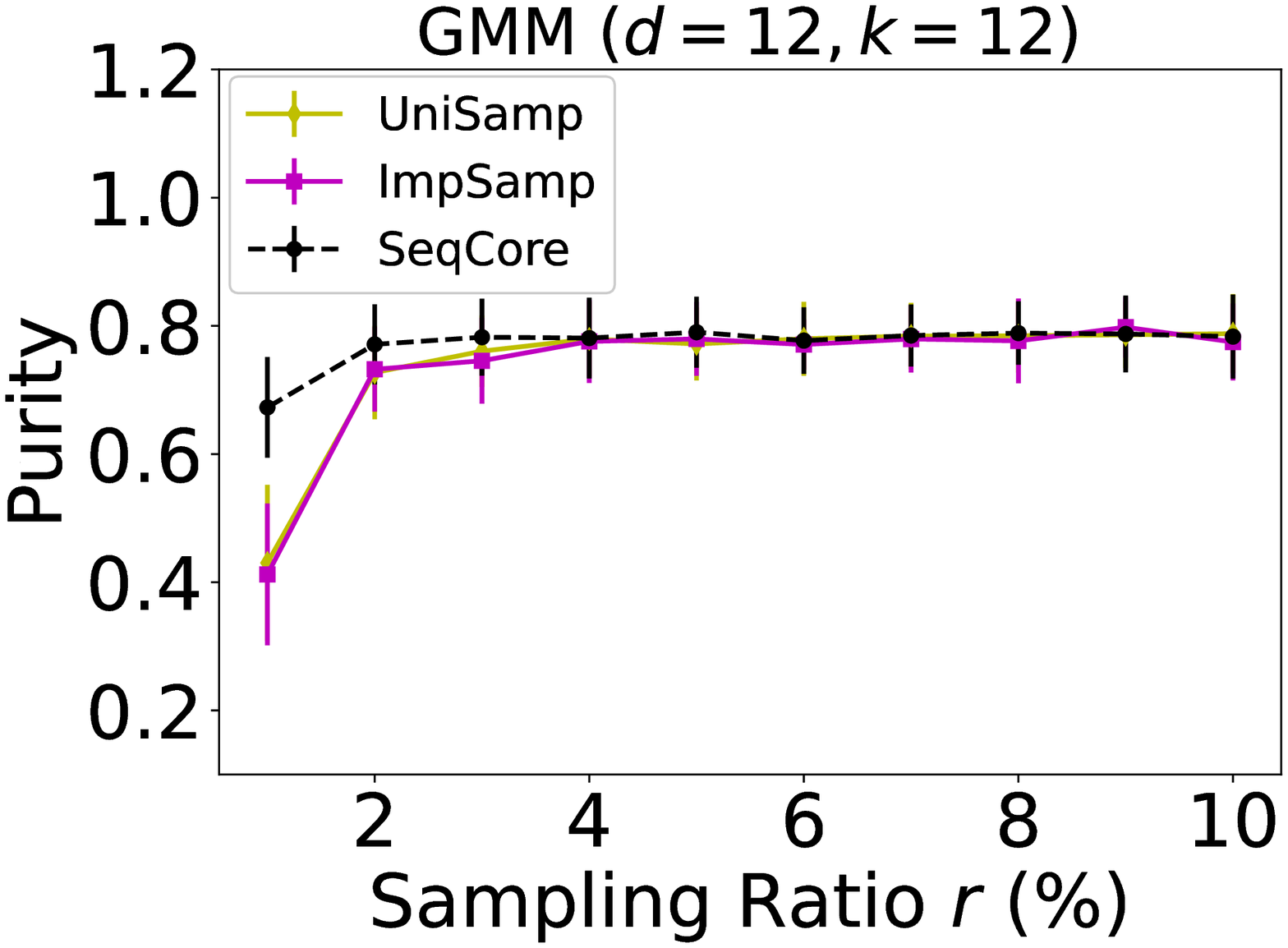}}
	\end{center}
	\vspace{-0.36in}
	\caption{The purity values on \textsc{synthetic dataset} for Gaussian Mixture Models. } 
	\label{fig:GMM_synth_purity}  
\end{figure*}

\section{Conclusions and Future Work}
Based on the simple observation of the locality property, we propose a novel sequential coreset framework for reducing the complexity of gradient descent algorithms and some relevant variants. Our framework is easy to implement and has provable quality guarantees. Following this work, it is interesting to consider building coresets for other optimization methods, such as the popular stochastic gradient descent method as well as second order methods.

\section{Acknowledgements}
The authors would like to thank Mingyue Wang and the anonymous reviewers for their helpful discussions and suggestions on improving this paper. 
This work was supported in part  by the Ministry of Science and Technology of China through grant 2019YFB2102200, the Anhui Dept. of Science and Technology through grant 201903a05020049, and Tencent Holdings Ltd through grant FR202003. 

\appendix
\section{Proof for Theorem 2}

Similar with the proof for Theorem 1, we fix a vector $\beta \in \mathbb{B}(\tilde\beta,R)$ and  $l\in\{1,\cdots,d\}$. We view $\nabla f_i(\beta)_{[l]}$ as an independent random variable for each $(x_i,y_i)\in Q_j$. 
Note that we have the bound 
$$\nabla f_i(\beta)_{[l]}\in [\min\limits_j\nabla f_j(\tilde\beta)_{[l]}-LR,\max\limits_j\nabla f(\tilde\beta)_{l}+LR]$$
by Assumption \ref{assump-lip}, where the length of the interval is at most $2M+2LR$. 
If keeping the partition of $P$ as Algorithm~\ref{alg: local coreset}, through the Hoeffding's inequality we know that 
\small{
\begin{eqnarray}
	\mathtt{Prob}\big[|\frac{1}{|{Q_j}|}\!\!\! \sum_{(x_i,y_i)\in {Q_j}}\!\!\!\!\!\nabla f_i(\beta)_{[l]}-\frac{1}{|P_j|}\!\!\! \sum_{(x_i,y_i)\in P_j}\!\!\!\!\!\nabla f_i(\beta)_{[l]}|\geq \frac{\sigma}{2}\big] \nonumber
	\end{eqnarray}
	\begin{eqnarray}
	\label{grad-bound}
        \leq \lambda 
	\end{eqnarray}
	}
	\normalsize

if the sample size $|Q_j|=\lceil 8(\frac{LR+M}{\sigma})^2\ln \frac{2}{\lambda} \rceil$. 
By taking the union bound of (\ref{grad-bound}) over $0\leq j \leq N$, we have
	\begin{equation}
	\label{equation:grad-core}
	\begin{split}
	\begin{aligned}
	&\quad n|\nabla\tilde F(\beta)_{[l]}-\nabla F(\beta)_{[l]}|\\
	&\leq \sum_{j=0}^{N}|\frac{|P_j|}{|Q_j|}\sum_{(x_i,y_i)\in Q_j}\nabla f_i(\beta)_{[l]}-\sum_{(x_i,y_i)\in P_j}\nabla f_i(\beta)_{[l]}|\\
	&\leq \sum_{j=0}^{N}|P_j|\frac{\sigma}{2}\\
	&= n\frac{\sigma}{2}\\
	\end{aligned}
	\end{split}
	\end{equation}
with probability at least $1-(N+1)\lambda$. 
That is, $|\nabla\tilde F(\beta)_{[l]}-\nabla F(\beta)_{[l]}| \leq \frac{\sigma}{2}$.

The, we apply the similar discretization idea. We build the grid with side length being equal to $\frac{\epsilon R}{\sqrt{d}}$, so as to obtain the representative points set $G$ with $|G|=O((\frac{2\sqrt{\pi e}}{\epsilon})^d)$. 
It is easy to see that (\ref{equation:grad-core}) holds for any $\beta\in G$ and any $l\in \{1,\cdots,d\}$ with probability at least $1-(N+1)d|G|\lambda$.

For any $\beta\in\mathbb{B}(\tilde\beta,R)$, we let $\beta'\in G$ be the representative point of the cell containing $\beta$. Then we have $\|\beta -\beta'\|\leq \epsilon R$. By Assumption \ref{assump-lip}, we have both 
\begin{eqnarray}
	\|\nabla F(\beta)_{[l]}-\nabla F(\beta')_{[l]}\| &\text{ and }& \|\nabla\tilde F(\beta)_{[l]}-\nabla\tilde F(\beta')_{[l]}\|\nonumber\\
	&\leq&\epsilon LR.\label{grad-grid-sum}
	\end{eqnarray}
Through the triangle inequality, we know 
\begin{eqnarray}
&&|\nabla\tilde F(\beta)_{[l]}-\nabla F(\beta){[l]}| \nonumber\\
&\leq&|\nabla\tilde F(\beta)_{[l]}-\nabla\tilde F(\beta')_{[l]}|+|\nabla\tilde F(\beta')_{[l]}-\nabla F(\beta')_{[l]}|\nonumber\\
&&+|\nabla F(\beta')_{[l]}-\nabla F(\beta)_{[l]}|\nonumber\\
&\leq& \frac{\sigma}{2}+2\times \epsilon LR,
\end{eqnarray}
where the last inequality comes from (\ref{equation:grad-core}) and (\ref{grad-grid-sum}). 
If letting $\epsilon=\frac{\sigma}{2LR}$, we have 
$$| \nabla \tilde F(\beta)_{[l]}-\nabla F(\beta)_{[l]}| \leq \sigma.$$

Recall that $|Q_j|=\lceil 8(\frac{LR+M}{\sigma})^2\ln \frac{2}{\lambda} \rceil$ and the success probability is $1-(N+1)d|G|\lambda$. Let $\lambda=\frac{1}{n(N+1)d|G|}$ and then we have the success probability being at least $1-\frac{1}{n}$.

Finally, the coreset size is 
\begin{eqnarray}
\sum^N_{j=0}|Q_j|=\tilde{O}\big((\frac{LR+M}{\sigma})^2\cdot d\big).
\end{eqnarray}
For the case that $\beta$ is restricted to have at most $k$ non-zeros entries in sparse optimizations, similar with Theorem \ref{local_coreset}, we know the coreset size can be reduced to be 
$\tilde{O}\big((\frac{LR+M}{\sigma})^2\cdot k\log d\big)$.

\section{Details in Section \ref{sec-beyond}}
\label{sec:beyond-app}
Recall that the formula (\ref{for-er-beyond}) of Section \ref{sec-beyond},
\begin{eqnarray}
F(\beta)=\frac{1}{n}\sum^n_{i=1}g(\beta, x_i, y_i)+\lambda \|\beta\|_p, \label{for-er-beyond}
\end{eqnarray}
where the function $g(\beta, x_i, y_i)$ is assumed to be differentiable and satisfy Assumption \ref{assump-lip}. Let
$$f(\beta, x_i, y_i)=g(\beta, x_i, y_i)+\lambda\|\beta\|_p.$$ 

First, we note that problem (\ref{for-er-beyond}) is usually solved by  generalizations of gradient descent method, such as subgradient methods~\cite{subgradient} and proximal gradient methods~\cite{DBLP:conf/pkdd/MosciRSVV10}. The key point is that these algorithms also enjoy the locality property described in Section~\ref{sec-our}. In (\ref{eqa: betabound}), we provide the upper and lower bounds of $f_i(\beta)$ ({\em i.e.,} $f(\beta, x_i, y_i)$) for the non-differentiable case (\ref{eq:reg}). For $0<p\leq 2$, by using the H\"{o}lder's inequality we obtain the similar bounds for non-differentiable case :
 $f_i(\beta) \in f_i(\beta_{\mathtt{anc}}) \pm  \left(\left(\frac{ \|\nabla g_i(\beta_{\mathtt{anc}})\|}{R}+\frac{L}{2}\right)R^2+\frac{\lambda d^{1/p-1/2}}{n}R\right)$.  
After replacing (\ref{eqa: betabound}) by these bounds, we can proceed the same analysis in Section~\ref{sec-ana} and attain a similar result with Theorem~\ref{local_coreset}. Here is the detailed analyse.

For any $\beta,\beta_{\mathtt{anc}}\in \mathbb{R}^d$ and $\|\beta - \beta_{\mathtt{anc}}\|_2\leq R$, we have
\begin{eqnarray}
\label{beyond-grad}
&&|f(\beta,x_i,y_i)-f(\beta_{\mathtt{anc}},x_i,y_i)|\nonumber\\
&\leq&|(g(\beta,x_i,y_i)-g(\beta_{\mathtt{anc}},x_i,y_i)| + \lambda|\|\beta\|_p-\|\beta_{\mathtt{anc}}\|_p|\nonumber\\
&\leq& \|\nabla g_i(\beta_{\mathtt{anc}})\|_2R + \frac{LR^2}{2} + \lambda\|\beta-\beta_{\mathtt{anc}}\|_p\nonumber\\
&\leq& \|\nabla g_i(\beta_{\mathtt{anc}})\|_2R + \frac{LR^2}{2} + \lambda d^{\frac{1}{p}-\frac{1}{2}}\cdot \|\beta-\beta_{\mathtt{anc}}\|_2, 
\end{eqnarray}
where the last inequality comes from H\"{o}lder's inequality. 
By (\ref{beyond-grad}), we have

$$f_i(\beta) \in f_i(\beta_{\mathtt{anc}}) \pm  \left(\left(\frac{ \|\nabla g_i(\beta_{\mathtt{anc}})\|}{R}+\frac{L}{2}\right)R^2+\ \lambda d^{1/p-1/2} R\right).$$

We define $M\coloneqq  \lambda d^{1/p-1/2} +\max\limits_{1\leq i\leq n}\|\nabla g(\beta_{\mathtt{anc}},x_i,y_i)\|$ and $m\coloneqq \!\!\! \min\limits_{\beta \in \mathbb{B}(\beta_{\mathtt{anc}},R)}F(\beta)$. Then we have the following theorem by using the same idea for Theorem \ref{local_coreset}.  
\begin{theorem}
	\label{local_coreset}
	With problability $1-\frac{1}{n}$, Algorithm \ref{alg: local coreset} returns a qualified coreset  $\mathcal{CS}_\epsilon(\beta_{\mathtt{anc}}, R)$ with size  $\tilde{O}\left(\left(\frac{H+MR+LR^2}{m}\right)^2\cdot\frac{d}{\epsilon^2}\right)$. Furthermore, when the vector $\beta$ is restricted to have at most $k\in \mathbb{Z}^+$ non-zero entries in the hypothesis space $\mathbb{R}^d$, the coreset size can be reduced to be $\tilde{O}\left(\left(\frac{H+MR+LR^2}{m}\right)^2\cdot\frac{k\log d}{\epsilon^2}\right)$. The runtime of Algorithm 1 is $O(n\cdot t_f)$, where $t_f$ is the time complexity for computing the loss $f(\beta, x, y)$. 	
\end{theorem}
\section{Assumption \ref{assump-lip} for GMM}
We show that the objective function of GMM training  satisfies Assumption \ref{assump-lip}. 
 In GMM,
$$f(x_i, \beta)=-\log\Big(\sum^k_{j=1}\omega_j N(x_i, \mu_j, \Sigma_j)\Big)$$
where $N(x_i, \mu_j, \Sigma_j)=\frac{1}{\sqrt{(2\pi)^D|\Sigma_j|}}\exp(-\frac{1}{2}(x_i-\mu_j)^T\Sigma^{-1}_j(x_i-\mu_j))$,$\beta \coloneqq [(\omega_1, \mu_1, \Sigma_1^{-1}), \hdots, (\omega_k, \mu_k, \Sigma_k^{-1})]$,
$\sum_{j=1}^k\omega_j=1$. For simplicity, we use $p_{ij}$ to denote $\mathcal{N}(x_i, \mu_j, \Sigma_j)$. We define  $\tilde m\coloneqq\min\limits_{1\leq l\leq k, 1\leq i\leq n}p_{il}$ and $\tilde M\coloneqq\max\limits_{1\leq l\leq k, 1\leq i\leq n}p_{il}$.

Here we assume that the Gaussian models are $\lambda$-semi-sphericial (following the assumption in~\cite{lucic2017training}), which means $\Sigma_j$ has eigenvalues bounded in $[\lambda,\frac{1}{\lambda}]$ for $1\leq j\leq k$. Also we assume that for any $j\in\{1,\cdots,k\}$, $\|x_i - \mu_j\|\leq r$ for some $r>0$. Therefore we know $\tilde m\geq (\frac{\lambda}{2\pi})^{\frac{D}{2}}e^{-\frac{r^2}{2\lambda}}$, $\tilde M\leq (\frac{1}{2\pi\lambda})^{\frac{D}{2}}e^{-\frac{\lambda r^2}{2}}. $ 
Also, we have the following equations:
\begin{eqnarray}
\frac{\partial f_i(\beta) }{\partial \mu_{j}}&=&\Sigma_{j}\left(x_{i}-\mu_{j}\right) \gamma_{i j};\\
\frac{\partial f_i(\beta) }{\partial \Sigma^{-1}_{j}}&=&\left(\Sigma_{j}-\left(x_{i}-\mu_{j}\right)\left(x_{i}-\mu_{j}\right)^{T}\right) \gamma_{i j};\\
\frac{\partial f_i(\beta) }{\partial \omega_{j}}&=&\omega_{j}^{-1} \gamma_{i j}.
\end{eqnarray}

Here $\{\gamma_{i1}, \cdots, \gamma_{ik}\}$ are the GMM responsibilities for $f_i$. Then we have $\gamma_{ij}=\frac{\omega_ip_{ij}}{\omega_1p_{i1}+\cdots+\omega_kp_{ik}}$. Thus we have 
\begin{eqnarray}
\label{bound-mu}
 \|\frac{\partial f_i(\beta) }{\partial \mu_{j}}\|_2&=&\gamma_{ij}\|\Sigma_{j}\left(x_{i}-\mu_{j}\right)\|_2\leq\frac{r}{\lambda};
\end{eqnarray}

\begin{eqnarray}
\label{bound-Sigma}
\|\frac{\partial f_i(\beta) }{\partial \Sigma^{-1}_{j}}\|_F
&=&\gamma_{i j}\|\Sigma_{j}-\left(x_{i}-\mu_{j}\right)\left(x_{i}-\mu_{j}\right)^{T}\|_F\nonumber\\
&\leq& \|\Sigma_j\|_F+\|\left(x_{i}-\mu_{j}\right)\left(x_{i}-\mu_{j}\right)^{T}\|_F\nonumber\\
&\leq&\frac{\sqrt{D}}{\lambda}+r^2;
\end{eqnarray}

\begin{eqnarray}
\label{bound-omega}
\frac{\partial f_i(\beta) }{\partial \omega_{j}}=\frac{p_{ij}}{\omega_1p_{i1}+\cdots+\omega_kp_{ik}}\leq \frac{\tilde M}{\tilde m};
\end{eqnarray}

From (\ref{bound-omega}), (\ref{bound-Sigma}) and (\ref{bound-mu}) we have 
$$\|\nabla f_i(\beta)\|\leq\sqrt{k\left(\frac{\tilde M^2}{\tilde m^2}+\frac{r^2}{\lambda^2}+(\frac{\sqrt{D}}{\lambda}+r^2)^2\right)}.$$

So $f_i$ is Lipschitz continous, {\em i.e.,}
\begin{eqnarray}
\label{GMM-bound}
|f_i(\beta_1) - f_i(\beta_2)|\leq \sqrt{k\left(\frac{\tilde M^2}{\tilde m^2}+\frac{r^2}{\lambda^2}+(\frac{\sqrt{D}}{\lambda}+r^2)^2\right)}\|\beta_1-\beta_2\| .
\end{eqnarray}
We can use (\ref{GMM-bound}) to replace the bound implied from Assumption~\ref{assump-lip} in Section \ref{sec-construction}, and thus the same analysis and results hold for GMM.

\bibliography{SeqCoreset-arxiv}
\bibliographystyle{plain}
\end{document}